\def\isarxiv{1}

\def\paperTitle{Exploring the Frontiers of Softmax: Provable Optimization, Applications in Diffusion Model, and Beyond}

\def\paperRTitle{\paperTitle} 

\def\paperAuthor{
Yang Cao\thanks{Wyoming Seminary.}
\and
Yingyu Liang\thanks{The University of Hong Kong.}
\and
Zhenmei Shi\thanks{University of Wisconsin-Madison.}
\and
Zhao Song\thanks{\texttt{magic.linuxkde@gmail.com}. Simons Institute for the Theory of Computing, University of California, Berkeley.}
}

\ifdefined\isarxiv
\documentclass[11pt]{article}
\usepackage[numbers]{natbib}
\else
\documentclass{article}

\fi

\ifdefined\isarxiv

\usepackage{amsmath}
\usepackage{amsthm}
\usepackage{amssymb}
\usepackage{algorithm}
\usepackage{subfig}
\usepackage{algpseudocode}
\usepackage{graphicx}
\usepackage{grffile}
\usepackage{wrapfig,epsfig}
\usepackage{url}
\usepackage{xcolor}
\usepackage{epstopdf}

\usepackage{bbm}
\usepackage{dsfont}

\else 

\usepackage{microtype}
\usepackage{graphicx}
\usepackage{subcaption}
\usepackage{booktabs} 

\usepackage{hyperref}

\usepackage{icml2026}

\usepackage{amsmath}
\usepackage{amssymb}
\usepackage{mathtools}
\usepackage{amsthm}

\fi

\ifdefined\isarxiv

\else
\usepackage{algpseudocode}
\fi
 
\allowdisplaybreaks

\ifdefined\isarxiv

\usepackage{tikz}
\usepackage{hyperref}  
\hypersetup{colorlinks=true,citecolor=blue,linkcolor=blue} 
\usetikzlibrary{arrows}
\usepackage[margin=1in]{geometry}

\else
\fi
 
\graphicspath{{./figs/}}

\theoremstyle{plain}
\newtheorem{theorem}{Theorem}[section]
\newtheorem{lemma}[theorem]{Lemma}
\newtheorem{definition}[theorem]{Definition}

\newtheorem{corollary}[theorem]{Corollary}

\newtheorem{assumption}[theorem]{Assumption}

\newtheorem{fact}[theorem]{Fact}
\newtheorem{remark}[theorem]{Remark}
\newtheorem{claim}[theorem]{Claim}

\newcommand{\wh}{\widehat}
\newcommand{\wt}{\widetilde}

\newcommand{\N}{\mathcal{N}}
\newcommand{\R}{\mathbb{R}}

\renewcommand{\d}{\mathrm{d}}

\renewcommand{\tilde}{\wt}

\renewcommand{\S}{\mathcal{S}}

\DeclareMathOperator*{\E}{{\mathbb{E}}}

\DeclareMathOperator{\poly}{poly}

\DeclareMathOperator{\vect}{vec}

\ifdefined\isarxiv

\else
\icmltitlerunning{\paperRTitle}

\fi

\begin{document}

\ifdefined\isarxiv

\date{}
\title{\paperTitle}
\author{\paperAuthor}

\else

\twocolumn[
  \icmltitle{\paperTitle}


  \icmlsetsymbol{equal}{*}

  \begin{icmlauthorlist}
    \icmlauthor{Firstname1 Lastname1}{equal,yyy}
    \icmlauthor{Firstname2 Lastname2}{equal,yyy,comp}
    \icmlauthor{Firstname3 Lastname3}{comp}
    \icmlauthor{Firstname4 Lastname4}{sch}
    \icmlauthor{Firstname5 Lastname5}{yyy}
    \icmlauthor{Firstname6 Lastname6}{sch,yyy,comp}
    \icmlauthor{Firstname7 Lastname7}{comp}
    \icmlauthor{Firstname8 Lastname8}{sch}
    \icmlauthor{Firstname8 Lastname8}{yyy,comp}
  \end{icmlauthorlist}

  \icmlaffiliation{yyy}{Department of XXX, University of YYY, Location, Country}
  \icmlaffiliation{comp}{Company Name, Location, Country}
  \icmlaffiliation{sch}{School of ZZZ, Institute of WWW, Location, Country}

  \icmlcorrespondingauthor{Firstname1 Lastname1}{first1.last1@xxx.edu}
  \icmlcorrespondingauthor{Firstname2 Lastname2}{first2.last2@www.uk}

  \icmlkeywords{Machine Learning, ICML}

  \vskip 0.3in
]

\printAffiliationsAndNotice{} 

\fi

\ifdefined\isarxiv
\begin{titlepage}
  \maketitle
  \begin{abstract}
    The softmax activation function plays a crucial role in the success of large language models (LLMs), particularly in the self-attention mechanism of the widely adopted Transformer architecture. However, the underlying learning dynamics that contribute to the effectiveness of softmax remain largely unexplored. As a step towards better understanding, this paper provides a theoretical study of the optimization and generalization properties of two-layer softmax neural networks, providing theoretical insights into their superior performance as other activation functions, such as ReLU and exponential. Leveraging the Neural Tangent Kernel (NTK) framework, our analysis reveals that the normalization effect of the softmax function leads to a good perturbation property of the induced NTK matrix, resulting in a good convex region of the loss landscape. Consequently, softmax neural networks can learn the target function in the over-parametrization regime. To demonstrate the broad applicability of our theoretical findings, we apply them to the task of learning score estimation functions in diffusion models, a promising approach for generative modeling. Our analysis shows that gradient-based algorithms can learn the score function with a provable accuracy. Our work provides a deeper understanding of the effectiveness of softmax neural networks and their potential in various domains, paving the way for further advancements in natural language processing and beyond.

  \end{abstract}
  \thispagestyle{empty}
\end{titlepage}

{\hypersetup{linkcolor=black}
\tableofcontents
}
\newpage

\else

\begin{abstract}
The softmax activation function plays a crucial role in the success of large language models (LLMs), particularly in the self-attention mechanism of the widely adopted Transformer architecture. However, the underlying learning dynamics that contribute to the effectiveness of softmax remain largely unexplored. As a step towards better understanding, this paper provides a theoretical study of the optimization and generalization properties of two-layer softmax neural networks, providing theoretical insights into their superior performance as other activation functions, such as ReLU and exponential. Leveraging the Neural Tangent Kernel (NTK) framework, our analysis reveals that the normalization effect of the softmax function leads to a good perturbation property of the induced NTK matrix, resulting in a good convex region of the loss landscape. Consequently, softmax neural networks can learn the target function in the over-parametrization regime. To demonstrate the broad applicability of our theoretical findings, we apply them to the task of learning score estimation functions in diffusion models, a promising approach for generative modeling. Our analysis shows that gradient-based algorithms can learn the score function with a provable accuracy. Our work provides a deeper understanding of the effectiveness of softmax neural networks and their potential in various domains, paving the way for further advancements in natural language processing and beyond.
\end{abstract}

\fi


\section{Introduction} \label{sec:intro}
Large Language Models (LLMs) like GPT4~\citep{aaa+23} from OpenAI and Claude 3~\citep{claude} from Anthropic have widely and profoundly changed the world. Some researchers believe they split human history into two parts: the Pre-LLM Era and the LLM Era. The LLMs have been widely used in human activities, such as education~\citep{ksk+23},   law~\citep{sun23}, finance~\citep{lwdc23}, bio-informatics~\citep{tte+23}, coding~\citep{hzl+24}, and even top AI conference reviews such as ICML, ICLR, NeurIPS, and AISTATS~\citep{liz+24}. 
To make LLMs successful, one of the cores of LLMs is the Transformer model architecture~\citep{vsp+17}, which has many advantages, including faster-parallelized inference rather than sequential inference like RNN~\citep{hs97}; being easy to scale up the model capacity to support the scaling laws in neural language models~\citep{kmh+20}, i.e. since the input and output dimension of each Transformer blocks is the same, we can stack an arbitrary number of layers as we want. 
The kernel design of the Transformer block is self-attention layers, where each block has many attention heads and each head has its three important private parameter matrices for key, query, and value operation. 
Many papers believe that the self-attention operation is the critical reason for emergent ability~\citep{wtb+22}, including in-context learning~\citep{oen+22,red24} and compositional ability to solve complex task~\citep{dls+24,lpc+24}. The Transformer is so successful and has been widely certified that this architecture can be adopted in many other modalities such as tabular data, image/video generation, e.g., the video diffusion model SORA~\citep{sora} using Transformer~\citep{px23} as its backbone.

\begin{table*}[!ht]
\caption{Comparing hidden neuron number $m$ in two-layer neural networks and training steps $\wh T$ are required under different activation functions to guarantee that, for any $\epsilon > 0$, with probability at least $0.99$, the training loss is smaller or equal to $\epsilon$. Here, $n$ is the number of training samples, and $\lambda$ is the smallest eigenvalue for the matrix of the neural tangent kernel, where $n >1 $ and $\lambda < 1$. We can see that the two-layer NN with softmax activation function requires almost the same number of neurons and training steps to converge as that with ReLU or exponential activation functions. More details: Theorem 3.6 in~\cite{mosw22} for ReLU; Theorem 1.1 in~\cite{gms23} for $\exp$; Corollary~\ref{cor:linear} in our paper for softmax. }
    \centering
    \begin{tabular}{cccc}
    \hline
         &  ReLU (\cite{mosw22}) & $\exp $ (\cite{gms23}) & Softmax (ours)\\
        \hline
        $m$ & $\Omega(\lambda^{-2} n^2 \log(n))$ & $\Omega(\lambda^{-2} n^{2+o(1)} \log^2(n))$ & $\Omega(\lambda^{-2} n^{2+o(1)} \log^2(n))$ \\
        $\wh T$ & $\Omega(\lambda^{-2} n^2 \log(n/\epsilon))$ & $\Omega(\lambda^{-2} n^{2+o(1)} \log(n/\epsilon))$ & $\Omega(\lambda^{-2} n^{2+o(1)} \log(n/\epsilon))$ \\
        \hline
    \end{tabular}
    \label{tab:compare}
\end{table*}

When we delve into the self-attention mechanism, we find the softmax function plays a crucial role~\citep{vsp+17}. It enables the model to focus on the most related information among input sequences by giving higher attention scores to the positions that are more relevant for the current position's representation and to capture dependencies between positions. \cite{clj20} find that softmax attention is more expressive and performs better than any convolutional layer. \cite{dsz23} exhibits softmax attention outperforms linear attention in most scenarios. 
Although the softmax function code has been executed every second on thousands of servers, there is a limited understanding of the following question:
\begin{center}
    \textit{$(*)~$ What is the learning mechanism that makes softmax so powerful? }
\end{center} 
To demystify the black box, in this paper, we analyze the Gradient Descent (GD) training dynamics for two-layer Neural Networks (NN) with softmax activation function for multi-dimensional regression, i.e., $F (W,x,a) \in \R^d$ and $\forall \ell \in \{1, \dots, d\}$,
\begin{align*}
F (W,x,a)_{\ell} := {{}{}{} m} \langle a_{\ell}, \exp( W^\top x ) \rangle \cdot \langle \exp(W^\top x ) , {\bf 1}_m \rangle^{-1},
\end{align*}
where $m$ is number of hidden neurons, $\exp(\cdot)$ is element-wise exponential function, $a_\ell, W$ are the first and second layer weights respectively and $x$ is the input data. 
Note that, the self-attention could be written as $F(W^KX, W^QX , W^V X)\in \R^{d \times n'}$, where $W^K, W^Q, W^V \in \R^{d\times d}$ denotes key, query, and value matrix and $X\in \R^{d \times n'}$ is a sequence of $n'$ tokens. Thus, studying the two-layer softmax network is the prerequisite to understanding self-attention. See more discussion in Section~\ref{sec:discussion}. 

There is a rich line of work studying two-layer NN learning trajectory under ReLU activation function~(\cite{ll18,dzps19,als19_icml,adh+19_icml,sy19,mmm19,syz21,bpsw21,mosw22,cb20,zgj21,llwa21,ccbg22} and many more) or exponential activation function from the latest work~\citep{gms23}. As far as we know, our work is the first to theoretically study the optimization and generalization of the two-layer softmax network and it is a first step on understanding the power of softmax.

One popular analysis method for studying over-parameterized NN is Neural Tangent Kernel (NTK)~\citep{jgh18}, where overparameterized networks are approximately linear models around their initialization, so the network training is almost convex.

To answer our $(*)$ question above, we adopt the powerful NTK analysis paradigm in this work. Our analysis shows that, because of the normalization effect of the denominator, the Neural Tangent Kernel induced by the softmax has a good perturbation property (Lemma~\ref{lem:perturb_w}), which means the loss landscape of the softmax version has a large convex region. 
Thus, the softmax NN requires almost the same number of neurons and training steps to fit the data and converge as ReLU or exponential NN, which is illustrated in Table~\ref{tab:compare} clearly (Theorem~\ref{thm:formal}). 
To demonstrate the broad applicability of our theoretical findings, we apply our analysis in a practical case study to show the generalization ability of softmax NN, where the task is learning score estimation functions in diffusion models with noisy labels, a promising approach for generative modeling, as we can smartly transfer it to a multi-dimensional regression task (Theorem~\ref{thm:diff_main}). Thus, we show that gradient-based algorithms can learn the score function with a provable accuracy. 

Our paper’s contributions are summarized as follows:
\begin{itemize}
    \item {\bf Softmax NTK:} We build up the first NTK analysis framework for two-layer NN with softmax activation function (Theorem~\ref{thm:formal}). Furthermore, our multi-dimensional regression setting is more general than previous work~\citep{mosw22,gms23} (ReLU and $\exp$) and can be degenerated to the linear regression setting.
    \item {\bf Diffusion Models Case Study:} We apply our results in learning score estimation functions in diffusion models with noisy labels to verify our analysis effectiveness (Theorem~\ref{thm:diff_main}). 
\end{itemize}

\section{Related Works}

\subsection{Neural Tangent Kernel}
Recently many studies show that the analysis of optimization and generalization for deep learning should be interwoven together. 
One line of work uses the first-order Tyler expansion to study sufficiently over-parameterized neural networks around its initialization like NTK, e.g.~\cite{mrh+18,zczg18,jgh18,ll18,als19_neurips,zg19,os19, lxs+19,nxl+19,y19,sy19,dll+19,als19_icml,cob19, ofls19, adh+19_icml, cg19, jt19,all19,os20, cfw+20,zczg20,gsjw20,bpsw21,mz22,mosw22,gms23,qss23,qmsyz23,qsy23,sy23,gqsw24,szz24} and more. Thus, the neural network optimization can be a convex problem. The NTK method has been widely used in different scenarios, such as preprocessing analysis~\citep{syz21,hswz22,als+23,scl+23,ssll23,ssl24,gqsw24}, federated learning \citep{lsy23}, LoRA adaptation~\citep{hwaz+21,xsw+24,smf+24} of LLMs~\citep{mwy+23}, and learning score
estimation functions in diffusion models~\citep{hrx24}. 

\subsection{Softmax and Attention in LLMs}
Recently, significant advances have been achieved in language modeling, particularly with the introduction of Transformer architectures and attention mechanisms~\citep{vsp+17}. Self-attention to capture long-range dependencies in text, revolutionizing the field of NLP, e.g., BERT~\citep{dclt19}, PaLM~\citep{cnd+22}, LLaMA~\citep{tli+23}, LLaMA 2~\citep{tms+23}, ChatGPT~\citep{chatgpt}, GPT4~\citep{aaa+23}, Claude 3~\citep{claude} and so on. Many works demonstrate the softmax is beyond other activation functions such as ReLU attention or linear attention in different aspects, e.g, approximation power~\citep{dsz23,sht24,nll+24,lls+24_grok}, prompt tuning~\citep{orst23}, in-context learning ability~\citep{gsx23,swxl23,cpm+24,cswy24}, compositional ability\citep{xsl24}.
Many works study to generalize the softmax into high order attention~\citep{as24_iclr} or to accelerate softmax computation~\citep{wlk+20,cld+20,szz+21,qsd+21,as23_neurips,bsz23,as24_arxiv,hjk+24,hlsl24,dsy24,syz24,gsy23_dp,gsyz23_quantum,kmz23,lls+24_conv}.
Another line of work analyzes a one-layer softmax network trained on the linear regression task~\citep{lsx+23,dlms23,dls23,csy24,gswy23,scwz24}, while our work studies a two-layer softmax setting.

\subsection{Diffusion Model}
Score-based generative diffusion models can generate high-quality image samples comparable to GANs which requires adversarial optimization~\citep{hja20,ssk+20,kll+24}. Based on the U-Net~\citep{rfb15}, stable diffusion can successfully generate business-used images. Based on the softmax-based self-attention~\citep{px23}, OpenAI released a video diffusion model, SORA~\citep{sora}, with a surprising performance.    
Another line of work study training diffusion models with a better theoretical guarantee~\citep{se19,se20,sk21,sgse20,sdme21,llt22,kdl22,sdcs23,lkb+23,cll23,cdd23,chzw23,sck23,yfz+23,bdd23,gkl24,ccl+24,glb+24,wcl+24,cks24}. In this work, we adapt our analysis in diffusion models.

\paragraph{Roadmap.}
We organize our paper as follows: 
In Section~\ref{sec:preliminary}, we introduce the notation system and problem setup. 
In Section~\ref{sec:main_results}, we present our main result, proving that a Softmax neural network with $\poly(nd)$ neurons can fit any training dataset consisting of $n$ $d$-dimensional samples for $d$-dimensional regression tasks. 
In Section~\ref{sec:technical_overview}, we outline the key techniques used to establish our main result. 
In Section~\ref{sec:extension_on_diffusion}, we extend our findings to Diffusion Models, demonstrating that Softmax neural networks can accurately learn score estimation even with noisy labels. 
Finally, in Section~\ref{sec:conclusion}, we conclude the paper.

\section{Preliminary} \label{sec:preliminary}
We first introduce some notations. Then, we will introduce our problem setup. 


{\bf Notations.}
We use ${\cal N}(\mu, \Sigma)$ to denote the Gaussian distribution with $\mu$ and covariance $\Sigma$.
For any positive integer $n$, we use $[n]$ to denote set $\{1,2,\cdots,n\}$. 

Let a vector $z\in\R^n$. We denote the $\ell_2$ norm as  $\|z\|_2:=( \sum_{i=1}^n z_i^2 )^{1/2}$, the $\ell_1$ norm as $\|z\|_1:=\sum_{i=1}^n |z_i| $, $\|z\|_0$ as the number of non-zero entries in $z$, $\| z \|_{\infty}$ as $\max_{i \in [n]} |z_i|$.  We use $z^\top$ to denote the transpose of a $z$. We use $\langle \cdot, \cdot \rangle$ to denote the inner product. Let $A \in \R^{n \times d}$, we use $\vect(A)$ to denote a length $nd$ vector. We denote the Frobenius norm as  $\|A\|_F:=( \sum_{i\in [n], j\in [d]} A_{i,j}^2 )^{1/2}$. 
For a function $f(x)$, $f$ is $L$-Lipschitz if $\| f(x) - f(y) \|_2 \leq L \cdot \| x - y \|_2$. Let ${\cal D}$ denote a distribution. We use $x \sim {\cal D}$ to denote that we sample a random variable $x$ from distribution ${\cal D}$. We use $\E[]$ to denote expectation and $\Pr[]$ to denote probability. We use p.s.d. to denote the positive-semidefinite matrix. 

As we have multiple indexes, to avoid confusion, we usually use $i,j \in [n]$ to index the training data, $\ell \in [d]$ to index the output dimension, $r \in [m]$ to index neuron number.  

\subsection{Model, Data, and Algorithm}

{\bf Models.} We consider a two-layer softmax neural network. The hidden layer has $m$ neurons, and we use the softmax function as the activation function, $F(W,\cdot,a):\R^{d_1} \rightarrow \R^{d_2} $ and $\forall \ell \in [d_2]$
\begin{align}\label{eq:model_softmax}
F (W,x,a)_{\ell} := {{}{}{} m} \langle a_{\ell}, \exp( W^\top x ) \rangle \cdot \langle \exp(W^\top x ) , {\bf 1}_m \rangle^{-1},
\end{align}
where $\exp(\cdot)$ is element-wise exponential function. We use $m$ as a normalization factor. Note that we can reduce the $d_2$ to $1$ for the linear regression setting. To simplify the proof, we let $d_1 = d_2$. Note that our proof can generalize to different $d_1, d_2$ easily. 

We only optimizing $W$ and not both $W$ and $a$ simultaneously as many previous works to simplify optimization, e.g., \cite{dzps19,sy19,mosw22}, where $x \in \R^d$ represents the input, $w_1, \cdots, w_m \in \R^d$ are weight vectors in the first layer, i.e., $W = [w_1, \cdots, w_m] \in \R^{d \times m}$, and $a_1, \cdots, a_d \in \R^m$ are weights in the second layer.
We can simplify the notation as $F (W, x)$ when the context is clear.

{\bf Data.} We have $n$ training data points ${\cal D}_n = \{ (x_i, y_i)\}_{i=1}^n$, where $x \in \R^d$ and $y \in \R^d$.\footnote{Our analysis can extend to $x_i \in \R^{d_1}$ and $y_i \in \R^{d_2}$ easily.} We denote $X = [x_1, \dots, x_n] \in \R^{d\times n}$ and $Y = [y_1, \dots, y_n] \in \R^{d\times n}$. We assume that $\|x_i\|_2 \le 1 $ and $\|y_i\|_2 \le 1 $, $\forall i \in [n]$.

{\bf Gradient Descent.} We use $e_{r}$ to denote a vector where the $r$-th coordinate is $1$ and everywhere else is $0$.
$\forall r\in [m], \forall \ell \in [d]$, 
we have $\frac{\partial F (W,x,a)_{\ell} }{\partial w_r} \in \R^d$ can be written as
\begin{align}\label{eq:relu_derivative}
& ~ \frac{\partial F (W,x,a)_{\ell} }{\partial w_r} \notag \\
= & ~ + {{}{}{}m} \langle a_{\ell} \circ e_{r}, \exp(W^\top x) \rangle \cdot \langle \exp(W^\top x) , {\bf 1}_m \rangle^{-1} x \notag \\ & ~
- {{}{}{}m}\langle a_{\ell}, \exp(W^\top x) \rangle \cdot \langle \exp(W^\top x) , {\bf 1}_m \rangle^{-2} \notag \\
& ~ \cdot \langle \exp(W^\top x), e_{r} \circ {\bf 1}_m \rangle x \notag \\
= & ~ + {{}{}{}m}\langle a_{\ell} \circ e_{r}, \S \rangle \cdot x  ~ - {{}{}{}m}\langle a_{\ell}, \S \rangle \cdot \langle \S, e_{r} \circ {\bf 1}_m \rangle x .
\end{align}

We have the softmax function $\S \in \R^{m \times n}$,  where $\S_i \in \R^{m}$ denotes $ \langle \exp(W^\top x_i) , {\bf 1}_m \rangle^{-1} \cdot \exp(W^\top x_i)$ and $\S_{i,r} \in \R$ denotes $ \langle \exp(W^\top x_i) , {\bf 1}_m \rangle^{-1} \cdot \exp(w_r^\top x_i)$, $\forall r\in [m], \forall i \in [n]$. For simplicity, we denote ${\alpha}_i$ as $\langle {\bf 1}_m , \exp({W}^{\top} x_i) \rangle$, $\exp_{i}$ as $\exp(W^\top x_i)$ and $\exp_{i, r}$ as $\exp(w_r^\top x_i)$, $\forall r\in [m], \forall i \in [n]$, when the context is clear. 
 
We use $W(\tau)$ to denote the weights of the first layer on the timestamp $\tau$ 
and similar for $\S(\tau)$ and $F(\tau)$ when the context is clear. Now, we introduce some necessary definitions used.

We first introduce the function over the whole training dynamic. 
\begin{definition}[$F(\tau)$, dynamic prediction]\label{def:u}
We define $F_i(\tau) \in \R^d$, for any timestamp $\tau$, as
\begin{align*}
F_{\ell,i}(\tau) := {{}{}{}m}
\langle a_{\ell} , \exp( W(\tau)^\top x_i ) \rangle \cdot \langle  \exp( W(\tau)^\top x_i ), {\bf 1}_m \rangle^{-1} . 
\end{align*}
Here $x_i \in \R^{d}$.
It can be rewritten as 
$
    F_{\ell,i}(\tau) = {{}{}{}m}\langle a_{\ell}, \S_i(\tau) \rangle.
$
\end{definition}
We consider $d$-dimensional MSE loss. 
\begin{definition}[Loss function over time]
We define the objective function ${\cal L}$ as below:
\begin{align*}
{\cal L} (W(\tau) ) := \frac{1}{2} \sum_{i\in [n]} \sum_{\ell \in [d]} ( F_{\ell,i} (\tau)- y_{\ell,i}  )^2 .
\end{align*}
\end{definition}
Thus, we define the gradient of $w$. 
\begin{definition}[$\Delta w_r(\tau)$]\label{def:Delta_w_r_at_time_t}
For any $r \in [m]$, we define $\Delta w_r(\tau) \in \R^d$ as below:
\begin{align*}
& ~ \Delta w_r(\tau) \\
: = & ~ \frac{\d {\cal L} (W(\tau)}{\d w_r(\tau)} \\
 = & ~ {{}{}{}m}\sum_{i=1}^n \sum_{\ell=1}^d  ( F_{\ell,i}(\tau) - y_{\ell,i} ) \\
 \cdot & ~ \Big( \langle a_{\ell} \circ e_{r}, \S_i(\tau) \rangle  - \langle a_{\ell}, \S_i(\tau) \rangle \cdot \langle \S_i(\tau), e_{r} \circ {\bf 1}_m \rangle  \Big) \cdot x_i
\end{align*}
where $\S_i(\tau) =\langle \exp(W(\tau)^\top x_i) , {\bf 1}_m \rangle^{-1} \cdot \exp(W(\tau)^\top x_i) $.
\end{definition}

We can simplify the gradient calculation by the fact $1 = \langle {\bf 1}_m, \S_i(\tau) \rangle$. Thus, we have the following claim. 
\begin{claim}\label{cla:Delta_w_r_at_time_t}
$
\Delta w_r(\tau) : = {{}{}{}m}\sum_{i=1}^n \sum_{\ell=1}^d  ( F_{\ell,i}(\tau) - y_{\ell,i} ) \cdot \Big(  (  \langle a_{\ell,r} \cdot {\bf 1}_m -  a_{\ell}, \S_i(\tau) \rangle ) \cdot \S_{i,r}(\tau) \Big) \cdot x_i.
$
\end{claim}

We use the gradient descent (GD) algorithm with the learning rate $\eta$ to train the network. 
As we only train the hidden layer $W$ and fix $a$, we have the following gradient update rule. 

\begin{definition}[Gradient descent]\label{def:update}
The gradient descent algorithm for optimizing the weight matrix $W$ is defined as:
\begin{align*}
W(\tau+1) = W(\tau) - \eta \Delta W(\tau),
\end{align*}
where $\Delta W(\tau) \in \R^{d \times m}$ and $\Delta w_r(\tau) \in \R^{d}$ is the $r$-th column of $\Delta W(\tau)$ defined in Definition~\ref{def:Delta_w_r_at_time_t}.
\end{definition}

\subsection{Neural Tangent Kernel}
Now, we are ready to introduce our key tools, Neural Tangent Kernel induced by the softmax. We define the kernel with respect to timestamp $\tau$.
\begin{definition}[Kernel function]\label{def:H_s}
For simplicity, we denote $\S(W^\top x_i)$ as $\S_i \in \R^{m}_{\geq 0}$ and $v_{\ell,r} = a_{\ell,r} \cdot {\bf 1}_m - a_{\ell} \in \R^m$.
We define the function (Gram matrix) $H : \R^{d\times m} \rightarrow \R^{nd \times nd}$ as following
\begin{align*}
H(W):= 
\begin{bmatrix}
H_{1,1} & H_{1,2} & \cdots & H_{1,d} \\
H_{2,1} & H_{2,2} &  \cdots & H_{2,d} \\
\vdots & \vdots & \ddots & \vdots \\
H_{d,1} & H_{d,2} & \cdots  & H_{d,d}
\end{bmatrix} , 
\end{align*}
and for each $\ell_1, \ell_2 \in [d]$, we have $H_{\ell_1,\ell_2} \in \R^{n \times n}$ is defined as
\begin{align*}
& ~ [H_{\ell_1,\ell_2}]_{i, j} (W) \\
:= & ~ \frac{1}{m} x_{i}^{\top} x_{j} \sum_{r=1}^{m} \langle v_{\ell_1, r}, \S_i \rangle \cdot {{}{}{}m}  \S_{i,r}  \cdot 
\langle v_{\ell_2,r}, \S_j \rangle \cdot {{}{}{}m}  \S_{j,r}.
\end{align*}
For any timestamp $\tau$, for simplicity, we denote $H(\tau) := H(W(\tau))$ and denote $H(0)$ as $H^*$. 
\end{definition}

Note that $H^*$ is a positive semi-definite matrix, and we denote its minimum eigenvalue as $\lambda :=\lambda_{\min}(H^*)$ and we assume $\lambda > 0$ as previous works~\citep{dzps19,als19_icml,als19_neurips}. 

{\bf Initialization.} We use symmetric initialization, which is widely used in previous works~\citep{dm20,dls22,mosw22,swl22,swl24}.
\begin{definition}[Symmetric initialization]\label{def:duplicate_weights}
For each $r\in [m/2]$, we initialize weights as below
\begin{itemize}
    \item We draw $w_{2r-1}$ from ${\cal N}(0, \sigma^2 I_d)$ and uniformly draw $a_{2r-1}$ from $\{-1,+1\}^d$. 
    \item We assign $a_{2r} = -a_{2r-1}$ and $w_{2r-1} = w_{2r}$.
\end{itemize}
\end{definition}
Due to symmetric initialization, we can easily see that $F(W(0),x) = 0, \forall x \in {\R^d}$.

\section{Main Results} \label{sec:main_results}

We first define a constant we used. 
\begin{definition}\label{def:B}
Let $C> 10$ denote a sufficiently large constant. 
We define parameter $B$ as follows  
$
 B:= \max\{ C\sigma \sqrt{ \log(nd/\delta) }, ~1 \}.
$
\end{definition}

Now, we are ready to present our main result, whose complete proof is in Appendix~\ref{sec:converge:mainresult}.
\begin{theorem}[Main result]\label{thm:formal}
Let 
    $\lambda=\lambda_{\min}(H^*)>0$,
    $m = \Omega( \lambda^{-2} n^2 d^2 \exp(18B)\log^2(nd/\delta)  )$,
    $\eta = 0.1 \lambda / (m n^2 d^2 \exp(16B))$, and
    $\wh{T} = \Omega( (m \eta \lambda)^{-1} \log(nd/\epsilon)  ) = \Omega( \lambda^{-2}n^2 d^2 \exp(16B) \cdot \log(nd/\epsilon) )$.
For any $\epsilon, \delta \in (0,0.1)$, 
after $\wh T$ iterations, with probability at least $1-\delta$, we have
$
    \| F(\wh T) - Y \|_F^2 \leq \epsilon.
$
\end{theorem}
If we fix $\delta$ and $\sigma$ in $B$ defined in the Definition~\ref{def:B}, since $\exp(\Theta(B)) = (nd)^{o(1)}$, we can simplify the $m = \Omega( \lambda^{-2} (nd)^{2+o(1)})$ and $\wh T = \Omega( \lambda^{-2} (nd)^{2+o(1)})$. 

The Theorem~\ref{thm:formal} means that as we have $\poly(nd)$ number of neurons and training steps, the softmax NN can fit any training datasets with $n$ number of $d$-dim training samples on $d$-dim regression task.  

\begin{corollary}\label{cor:linear}
Consider the 1-dimension linear regression setting, i.e., $d_1 =d$ and $d_2 = 1$. Let
    $\lambda=\lambda_{\min}(H^*)>0$,
    $ m = \Omega( \lambda^{-2} n^2 \exp(18B)\log^2(n/\delta))$, 
    $\eta = 0.1 \lambda / (m n^2 \exp(16B))$, and
    $\wh T = \Omega( (m \eta \lambda)^{-1} \log(n/\epsilon)  ) = \Omega( \lambda^{-2}n^2 \exp(16B) $ $\cdot \log(n/\epsilon) ) $.
For any $\epsilon, \delta \in (0,0.1)$, 
after $\wh T$ iterations, with probability at least $1-\delta$, we have
$
    \| F(\wh T) - Y \|_2^2 \leq \epsilon.
$
\end{corollary}
\begin{proof}
    Directly follow Theorem~\ref{thm:formal}.
\end{proof}

As shown in Table~\ref{tab:compare}, our two-layer softmax network needs the same number of training steps $\wh T$ and number of neurons $m$ as two-layer ReLU networks or two-layer exponential networks.

\section{Technical Overview} \label{sec:technical_overview}


We first show a key Lemma below, showing that the weight $w$ perturbation will not change the Neural Tangent Kernel too much.

\begin{lemma}[Weight value perturbation $\Rightarrow$ kernel value perturbation]\label{lem:perturb_w}
Let $R \in (0,0.01)$.
If the following conditions hold
\begin{itemize}
    \item Let $\wt{W} = [\wt{w}_1, \cdots, \wt{w}_m]\in \R^{d\times m}$, where $\wt{w}_1, \cdots, \wt{w}_m$ are i.i.d. draw from ${\N}(0,\sigma^2 I_d)$. 
    \item Let $W = [w_1, \cdots, w_m] \in \R^{d\times m}$ and satisfy $\| \wt{w}_r - w_r \|_2 \leq R$
    for any $r\in [m]$. 
\end{itemize}
Then, with probability at least $1-\delta$, we have
$
\| H (W) - H(\wt{W}) \|_F \leq R nd \exp(10B).
$
\end{lemma}
Please see Appendix~\ref{sec:intialization_and_perturbation:bound_changes_H_w} for the proof of Lemma~\ref{lem:perturb_w}. 
We can see that the kernel matrix has a small perturbation when the weights $w$ perturb. Note that in Lemma 4.2 of \cite{mosw22}, they have $\| H (W) - H(\wt{W}) \|_F \leq 2Rn$ for the ReLU activation function and in Lemma 6.7 of \cite{gms23}, they have $\| H (W) - H(\wt{W}) \|_F \leq 3Rn^{1+o(1)} $ for the $\exp$ activation function. When we consider the 1-dimension linear regression task, we have $\| H (W) - H(\wt{W}) \|_F \leq R n^{1+o(1)} $, which is almost the same as the other two cases. 

\begin{remark}\label{rem:no_con}
    In the proof of Lemma~\ref{lem:perturb_w_formal}, we do not use concentration bound as previous work~\citep{sy19,mosw22,gms23}. The reason is that we consider the worst case. In general, $\E[H(W) - H(\wt W)] \neq \mathbf{0}_{nd\times nd} $. Thus, using the concentration bound may not gain any benefits. 
\end{remark}

Based on Lemma~\ref{lem:perturb_w},
we can use math induction to finish the proof of our main Theorem. We show the induction statement below.  

\begin{lemma}[Induction]\label{lem:induction}
Let $\tau$ be a fixed integer. Assume the same condition as Theorem~\ref{thm:formal}. Let $D$ be defined as Definition~\ref{def:D} and $D<R$. If the following conditions hold 
\begin{itemize}
    \item {\bf Weights Property.} $\| w_r(i) - w_r(0)\|_2 \leq  R$, $\forall i \in [\tau]$ 
    \item {\bf Loss Property.}  $\| F(i) - Y \|_F^2 \leq \| F(0) - Y \|_F^2\cdot (1-m\eta \lambda/2)^i$,  $\forall i \in [\tau]$ 
    \item {\bf Gradient Property.} $\eta \| \Delta w_r(i) \|_2 \leq 0.01$ for all $r \in [m]$, $\forall i \in [\tau]$ 
\end{itemize}
Then, for $\tau+1$ and $\forall r\in [m]$, we have 
\begin{itemize}
    \item {\bf Weights Induction.} $\| w_r(\tau+1) - w_r(0) \|_2 \leq D.$
    \item {\bf Loss Induction.} $\| F (\tau+1) - Y \|_F^2 \leq ( 1 - m \eta \lambda / 4 )^{\tau+1} \cdot \| F (0) - Y \|_F^2.$
    \item {\bf Gradient Induction.} $\eta \| \Delta w_r(\tau+1) \|_2 \leq 0.01, \forall r \in [m].$
\end{itemize}
\end{lemma}
Please refer to Appendix~\ref{app:induction_weights}, Appendix~\ref{app:induction_loss} and Appendix~\ref{app:induction_gradient} for the proof of weights, loss, gradient induction in Lemma~\ref{lem:induction} respectively. 

Lemma~\ref{lem:induction} means that, at a fixed timestamp $\tau$, if the weights $w(\tau)$ is close to its initialization, the loss is decreasing, and the gradient is also small, then we can conclude at timestamp $\tau + 1$, these conditions still hold as local convexity proved by Lemma~\ref{lem:perturb_w}. Thus, after checking the initial condition, we can conclude Theorem~\ref{thm:formal}.

\subsection{Technical Novelty and Comparison to the Existing Literature} 
In this work, as we consider the softmax activation function, the denominator term will also contribute to gradient calculation. Handling the denominator poses many technical challenges, where these challenges are unique to our setting and not presented in previous settings as ReLU \citep{sy19}, or $\exp$ \citep{gms23} activation function. In detail, in the gradient calculation, we need new loss decomposition Lemma~\ref{lem:rewrite_shrinking_one_step_v2} to split the loss into $\|F(\tau+1) - Y \|_F^2 = \| F(t )-Y \|_F^2 + C_0 + C_1 + C_2 + C_3$. 
Then, we need to bound these new terms in Lemma~\ref{lem:bound_c0} for $C_0$, Lemma~\ref{lem:bound_c1} and Claim~\ref{cla:C1} for $C_1$, Claim~\ref{cla:C2} for $C_2$ and Claim~\ref{cla:C3} for $C_3$, where all these Lemmas are novel and non-trivial. We refer readers to Appendix~\ref{app:induction_loss} for more details.

\section{Extension on Diffusion} \label{sec:extension_on_diffusion}

Now, we apply our results in learning score estimation functions in diffusion models with noisy labels. We introduce problem setup in Section~\ref{sec:diff_setup} and show our results in Section~\ref{sec:diff_main}. 

\subsection{Preliminary of Diffusion}\label{sec:diff_setup}
In this section, we briefly introduce the diffusion model proposed in~\cite{ssk+20}.

\textbf{Forward Process.}
During the forward process, we progressively inject the noise into the original data distribution, which can be characterized by the following Stochastic Differential Equation (SDE)~\citep{se20,hja20}:
\begin{align}\label{eq:diffusion}
    \mathrm{~d} x(t)=-\frac{1}{2} g(t) x(t) \mathrm{~d} t+\sqrt{g(t)} \d  B_{t}, \quad x(0) \sim p_{0},
\end{align}
where $x(t)$ is the data at the diffusion process time $t$, $g(t) > 0$ is a deterministic weighting function; and $(B_t )_{t\ge0}$ is a standard $d$-dimensional Brownian motion/noise. 
The $p_0$ represents the original/target data distribution that we learn, and we only have few number of accesses to it, i.e., $n$ times. We denote $p_t$ as the distribution of $x(t)$ at diffusion process time $t$. Then, we can write the explicit solution to Eq.~\eqref{eq:diffusion} as
\begin{align*}
    x(t)= & ~ e^{-\int_{0}^{t} \frac{1}{2} g(s) \d  s} x(0) \\
    & ~ + 
    e^{-\int_{0}^{t} \frac{1}{2} g(s) \d  s} \int_{0}^{t} e^{\int_{0}^{s} \frac{1}{2} g(u) \d  u} \sqrt{g(s)} \d  B_{s}.
\end{align*}

\textbf{Backward Process.}
We denote $y(t) = x (T-t)$ to reverse the forward process in time~\citep{hp86,fol05,ccgl21} that transforms noise into samples from the target distribution. We have a backward process associated to Eq.~\eqref{eq:diffusion} as:
\begin{align}\label{eq:diffusion_back}
    \d  y(t) = & ~ (\frac{1}{2} g(T-t) y(t)+g(T-t) \nabla \log p_{T-t} (y(t) ) ) \d \notag \\
    & ~ + \sqrt{g(T-t)} \d  \bar{B}_{t}, \quad y(0) \sim q_{0}. 
\end{align}
where $(\bar{B}_{t} )_{t \ge 0}$ is another $d$-dim Brownian motion/noise. Following the literature, we call $\nabla \log p_{t} (\cdot)$ as ``score function''~\citep{ssk+20}. We have $q_0$ as the initial distribution of the backward process and the score function 
$\nabla \log p_{t} (\cdot)$  as the gradient of log density of $x(t)$. 

However, In practice, Eq.\eqref{eq:diffusion_back} cannot be directly used as both the score function and the distribution $p_T$ are unknown. To solve the problem, we (1) randomly select a noise distribution as the initial distribution of the backward process $p_T$; (2) replace the ground-truth score function $\nabla \log p_{t} (x(t))$ by an estimator $s_\theta (x(t),t)$. The parameterized estimator $s_\theta$ is learned by a neural network such as U-Net~\citep{hja20,rbl+22} and Transformer~\citep{px23}. Thus, we obtain a practically implementable approximation of the backward SDE:
\begin{align*}
    \mathrm{~d} y(t)= & ~ (\frac{1}{2} g(T-t) y(t)+g(T-t) s_{\theta} (y(t), t ) ) \d  t \\
    & ~ +\sqrt{g(T- t)} \d  \bar{B}_{t}, \quad y(0) \sim \mathcal{N} (0, I_{d} ),
\end{align*}
which can be used for sampling/data generation~\citep{se20,chzw23,ccl+23}

\textbf{Score Matching.}
When estimating the score function, we usually use $L_2$ loss between the estimated and actual score:
\begin{align}
    \min_{\theta} \frac{1}{T} \int_{0}^{T} \lambda(t) \mathbb{E} [ \|s_{\theta} (x(t), t )- \nabla \log p_{t} (x(t) ) \|_{2}^{2} ] \d  t , \label{eq:score}
\end{align}
where $\lambda(t)$ is the weighting function that captures time inhomogeneity. As the hardness of estimate $\nabla \log p_{t}$ term in Eq.~\eqref{eq:score}, equivalently, we
minimize the following denoising score matching~\citep{vin11}:
\begin{align}\label{eq:diff_final}
    & ~ \min _{\theta} \frac{1}{T-T_{0}}
    \cdot \int_{T_{0}}^{T} \lambda(t) \notag \\
    & ~ \mathbb{E} [ \|s_{\theta} (x(t), t ) - \nabla \log p_{t \mid 0} (x(t) \mid x(0) ) \|_{2}^{2} ] \d  t .  
\end{align}

In practice, the estimator of the score function is parameterized by a neural network, and we have the following sampling procedure for any $i \in [n]$,
\begin{align*}
    x(0)_i \sim p_0, \quad t_i \sim \mathrm{Unif}(0,T), \quad x(t_i)_i \sim p_{t_i|0}(\cdot | x(0)_i),
\end{align*}
and we get the training dataset $ \{x(0)_i, (t_i,  x(t_i)_i) \}_{i=1}^{n}$, where $x(0)_i \in \R^d$ and $(t_i,  x(t_i)_i) \in \R^{d+1}$. We denote $x(0)$ as the noisy label and $\E[x(0)|x(t)]$ as the true label.
For simplicity, we denote $x(0)_i$ as $y_i \in \R^{d}$ and $(t_i,  x(t_i)_i) $ as $x_i\in \R^{d+1}$ and the training dataset as ${\cal D}_n = \{ (x_i, y_i)\}_{i=1}^n$. Here, $y$ denotes the image from a dataset, and $x$ denotes the noised image with its diffusion process time $t$.

\textbf{Neural Network Parameterization.}
Recall that we consider a two-layer network with softmax activation function as the diffusion model in Eq.~\eqref{eq:model_softmax}, satisfying $\forall \ell \in [d]$, $F (W,x,a)_{\ell} = {{}{}{} m }\langle a_{\ell}, \exp( W^\top x ) \rangle \cdot \langle \exp(W^\top x ), {\bf 1}_m \rangle^{-1}$. Note that we do not train the top-layer weights $a$, so we can denote it as $F_{nn}(W,x)$. 

Then, similar as~\citep{hja20,hrx24}, our loss function Eq.~\eqref{eq:diff_final} can be rewrite as
\begin{align*}
    \min_{W} \mathcal{L}(W):=\frac{1}{2} \sum_{j=1}^{N} \|F_{nn} (W,x_j) - y_j \|_{2}^{2} .
\end{align*}
We denote the target function as 
$
    F_{*}(t, x(t)):=\mathbb{E} [y \mid (t, x(t)) ].
$
Let $\mathcal{H}$ be the reproducing Hilbert space (RKHS) induced by the NTK~\citep{cdvtu10,jgh18} and let $F_\mathcal{H}$ in the RKHS $\mathcal{H}$ such that $ \|F_{\mathcal{H}} \|_{\mathcal{H}}^{2} \leq R_{\mathcal{H}}$.

\subsection{Main Result of Diffusion}\label{sec:diff_main}
We first introduce some natural assumptions we used. 
\begin{assumption}\label{ass:target_density_function}
    Based on normalization, we assume $\|y_i\|_2 \le 1, \|x_i\|_2 \le 1, \forall i \in [n]$.
\end{assumption}

\begin{assumption}\label{ass:smallest_eigenvalue_probability_bound_38}
Assume $\lambda=\lambda_{\min}(H^*)>0$.
\end{assumption}

\begin{assumption}\label{ass:function_g}
    The function $g$ is almost everywhere continuous and bounded on $[0, \infty)$.
\end{assumption}

\begin{assumption}\label{ass:function_f_lipschitz_x}
    For all $(t,x(t)) \in (0, \infty ) \times \mathbb{R}^{d}$, the function $F_{*}(t,x(t))$ is $\beta_{x}$-Lipschitz 
    in $x$, i.e., $ \|F_{*}(t,x(t))-F_{*} (t,x'(t)) \|_{2} \leq \beta_{x} \|x(t)-x'(t) \|_{2}$.
\end{assumption}

We denote $A (R_{\mathcal{H}}):=c_{1} \Lambda (\frac{\sqrt{R_{\mathcal{H}}}}{\Lambda} )^{-\frac{2}{d}} \log  (\frac{\sqrt{R_{\mathcal{H}}}}{\Lambda} )$ and $\Lambda=O (\sqrt{d})$ and 
\begin{align*}
\Gamma_{\delta} := &~ \left( \frac{ 2 d^2  A (R_{\mathcal{H}}) }{\lambda} \log ^{3 / 2} (\frac{e (d n)^{3 / 2} A (R_{\mathcal{H}})}{\lambda} )  +\frac{1}{\sqrt{n}} \right)^{2} \\
& ~ +\frac{d^{2} A^{2} (R_{\mathcal{H}}) }{\lambda^{2}}(\log (1 / \delta)+\log (\log n)).
\end{align*}

\begin{assumption}[Assumption 3.11 in \cite{hrx24}]\label{ass:kernel_regression_error_bound}
Fix any $F_{\mathcal{H}} \in \mathcal{H}$ with $ \|F_{\mathcal{H}} \|_{\mathcal{H}}^{2} \leq R_{\mathcal{H}}$ and assume labels are generated as $\tilde{y}_{j}=F_{\mathcal{H}} (x_j)+\epsilon_{j}$. Suppose $\tilde{F}_{ntk}(\gamma(\widehat{T}), \cdot )$ is obtained by GD-trained kernel regression with the number of iterations $\widehat{T}$. We assume there exists $\epsilon$ such that
\begin{align*}
& ~ \frac{1}{T} \int_{0}^{T} \mathbb{E} [ \tilde{F}_{ntk}(\gamma(\widehat{T}), (t, x(t) ) ) -F_{\mathcal{H}} (t, x(t) ) \|_{2}^{2} ] \mathrm{d} t \\
\leq & ~ \epsilon(n, \widehat{T}), 
\end{align*}
and $\epsilon(n, \widehat{T})  \rightarrow 0$ as $n  \rightarrow \infty .$
\end{assumption}

Now, we are ready to present our main Theorem for diffusion. 
\begin{theorem}[Main results of score estimation and generalization]\label{thm:diff_main}
Suppose Assumptions~\ref{ass:target_density_function},~\ref{ass:smallest_eigenvalue_probability_bound_38},~\ref{ass:function_g},~\ref{ass:function_f_lipschitz_x} hold and we set {{}{}{}$ m = \Omega( \lambda^{-2} n^3 d^3 \exp(18B)\log^2(nd/\delta)  )$} and {{}{}{} $\eta = 0.1 \lambda / (m n^2 d^2 \exp(16B))  $}. 
Moreover, suppose early stopping $\wh{T}$ satisfies Assumption~\ref{ass:kernel_regression_error_bound} with corresponding $\epsilon(n, \wh{T})$. Then for large enough $R_{\mathcal{H}}$, with probability at least $1-\delta$, it holds that
\begin{align*}
& ~ \frac{1}{T} \int_{0}^{T} \lambda(t) \mathbb{E} [ \|s_{W(\wh{T})} (t, x(t) )-\nabla \log p_{t} (X_{t} ) \|_{2}^{2} ] \d  t \\
\leq & ~  O( \frac{1}{\lambda \sqrt{n}}  + \epsilon(n, \wh{T}) + d A^{2} (R_{\mathcal{H}}) \\
& ~ + d A (R_{\mathcal{H}})
+\sqrt{d A (R_{\mathcal{H}}) \Gamma_{\delta}}+\Gamma_{\delta} ).
\end{align*}
\end{theorem}
Please refer to Appendix~\ref{app:diff_main} for the complete proof. Here, we provide a proof sketch.
\begin{proof}[Proof sketch of Theorem~\ref{thm:diff_main}]
In Theorem~\ref{thm:kernel_perturbation_to_prediction}, we show the ``equivalence'' between softmax NN learning and corresponding neural tangent kernel regression, i.e., the gap between them is always small. Then, we can borrow the generalization ability of kernel regression to the generalization ability of two-layer softmax NN.  
On the other hand, by Claim~\ref{cla:decompose}, we can decompose the loss into a coupling gap, a label mismatch gap, an early stopping gap, and an approximation gap. By using our Theorem~\ref{thm:formal}, Theorem~\ref{thm:kernel_perturbation_to_prediction} with some tools from~\cite{hrx24}, we finish the proof.   
\end{proof}
From Theorem~\ref{thm:diff_main}, we know that, under some natural assumptions, the GD algorithm trained two-layer softmax NN can learn a provable accuracy on the score estimation functions in the diffusion model with noisy labels. We use this practical case study to demonstrate the broad applicability of our theoretical findings.

\section{Conclusion} \label{sec:conclusion}
This paper provides a theoretical analysis of the optimization and generalization properties of two-layer neural networks with the softmax activation function.
We apply our results in learning score estimation
functions in diffusion models with noisy labels to verify our analysis effectiveness.
Our findings contribute to a deeper understanding of the power of softmax neural networks and their potential for self-attention, advanced LLMs, and generative modeling.


\ifdefined\isarxiv
\section*{Acknowledgement}
Research is partially supported by the National Science Foundation (NSF) Grants 2023239-DMS, CCF-2046710, and Air Force Grant FA9550-18-1-0166. The authors would like to thank Yufa Zhou for his helpful suggestions and feedback.
\else

\fi

\ifdefined\isarxiv
\else
\section*{Impact Statement}

This paper presents work whose goal is to advance the field of Machine
Learning. There are many potential societal consequences of our work, none
which we feel must be specifically highlighted here.
\fi

\ifdefined\isarxiv
\bibliographystyle{alpha}
\bibliography{ref}
\else
\bibliographystyle{icml2026}
\bibliography{ref}
\fi


\newpage
\onecolumn
\appendix

\begin{center}
    \textbf{\LARGE Appendix }
\end{center}

{\bf Roadmap.}
In Section~\ref{sec:def}, we introduce some definitions that will be used in the proof. In Section~\ref{sec:basic_concentration}, we provide the basic concentration. In Section~\ref{sec:induction}, we provide the proof of our inductions. In Section~\ref{sec:induction_for_weight}, we establish a bound for the weight of induction Part 1. In Section~\ref{sec:induction_for_loss}, we establish a bound for the loss of induction Part 2. In Section~\ref{sec:ntk_regression}, we introduce the NTK regression. In Section~\ref{sec:diffusion}, we introduce the diffusion.
In Section~\ref{sec:discussion}, we discuss the potential implications of our results for popular frameworks such as attention mechanisms and feature learning. 
\section{Definition}\label{sec:def}

\begin{claim}[Restatement of Claim~\ref{cla:Delta_w_r_at_time_t}]\label{cla:Delta_w_r_at_time_t_formal}
We have
\begin{align*}
\Delta w_r(\tau) : = {{}{}{}m}\sum_{i=1}^n \sum_{\ell=1}^d  ( F_{\ell,i}(\tau) - y_{\ell,i} ) \cdot \Big(  (  \langle a_{\ell,r} \cdot {\bf 1}_m -  a_{\ell}, \S_i(\tau) \rangle ) \cdot \S_{i,r}(\tau) \Big) \cdot x_i
\end{align*}
\end{claim}
\begin{proof}[Proof of Claim~\ref{cla:Delta_w_r_at_time_t}]
We can show that
\begin{align*}
\Delta w_r(\tau) / {{}{}{}m} 
= & ~ \sum_{i=1}^n \sum_{\ell=1}^d  ( F_{\ell,i}(\tau) - y_{\ell,i} ) \cdot ( \langle a_{\ell} \circ e_r - a_{\ell} \cdot \S_{i,r}(\tau) , \S_i(\tau) \rangle )   x_i \\
= & ~ \sum_{i=1}^n \sum_{\ell=1}^d  ( F_{\ell,i}(\tau) - y_{\ell,i} ) \cdot \Big(  ( a_{\ell,r} - \langle a_{\ell}, \S_i(\tau) \rangle ) \cdot \S_{i,r}(\tau) \Big) \cdot x_i \\
= & ~ \sum_{i=1}^n \sum_{\ell=1}^d  ( F_{\ell,i}(\tau) - y_{\ell,i} ) \cdot \Big(  \underbrace{  \langle a_{\ell,r} \cdot {\bf 1}_m -  a_{\ell} }_{m \times 1}, \underbrace{ \S_i(\tau) }_{m \times 1} \rangle  \cdot \S_{i,r}(\tau) \Big) \cdot x_i,
\end{align*} 
where the first step follows from the definition of $\Delta w_r(\tau)$, the second step follows from $\langle a_{\ell} \circ e_r , x \rangle = a_{\ell,r} x_r$, 
and the last step is due to the Fact~\ref{fac:inner_product}.
\end{proof}

We present the following definition to simplify the notation. 
\begin{definition}\label{def:D}
We define $D$ 
{{}{}{}
\begin{align*}
D: =4 m^{-1}\lambda^{-1} \exp(3B) \sqrt{nd} \cdot    \|F(0)-Y \|_F
\end{align*}
}
\end{definition}

\begin{fact}\label{fac:squared_euclidean_distance}
For any vectors $ u, v \in \mathbb{R}^n $, the squared Euclidean distance between $u$ and $v$ can be expressed as:

\begin{align*}
    \| u - v \|_2^2 = \| u \|_2^2 - 2 u^\top v + \| v \|_2^2.
\end{align*}
\end{fact}

\begin{fact}\label{fac:inner_product}
Let $\mathbf{1}_m$ be a vector of dimension $m$ consisting of all ones, and $\S_i(\tau) \in \R^m_{\geq 0}$ be the indicator of some function $\tau$ at position $i$. We have:
\begin{align*}
    1 = \langle {\mathbf{1}_m}, \S_i(\tau) \rangle
\end{align*}

\end{fact}

\begin{fact}\label{fac:inequality}
For any real number $|x| \leq 0.1$, the following inequality holds:
\begin{align*}
    (1-x)^{1/2} \leq 1-0.5 x
\end{align*}
\end{fact}

\begin{fact}\label{fac:exp}
For any real number $|x| \leq 0.1$,
we have 
\begin{align*}
    |\exp(x)-1| \leq 2 |x|
\end{align*}
\end{fact}

\begin{fact}\label{fac:geometric_series}
For any $x \in (0,0.1)$, we have
\begin{align*}
\sum_{i=0}^{\infty} x^i \leq \frac{1}{1-x}
\end{align*}
\end{fact}

\begin{fact}\label{fac:exp_approximation}
For any $|x| \leq 0.01$, we have 
\begin{align*}
    \exp(x) = 1+x + \Theta(1) x^2
\end{align*}
\end{fact}

We state the standard Hoeffding inequality,
\begin{lemma}[Hoeffding inequality \cite{h63}]\label{lem:hoeffding}
If the below conditions are true
\begin{itemize}
    \item Let $x_1, \cdots, x_n$ denote $n$ independent variables
    \item $x_i \in [\alpha_i,\beta_i]$, for all $i \in [n]$
    \item  Let $ x= \sum_{i=1}^n x_i$.
\end{itemize}
 Then we have
\begin{align*}
\Pr[ | x - \E[x] | \geq t ] \leq 2\exp \left( - \frac{2t^2}{ \sum_{i\in [n]} (\beta_i - \alpha_i)^2 } \right).
\end{align*}
\end{lemma}

\begin{lemma}[Hanson-Wright inequality~\cite{hanson,hanson2}]\label{lem:hanson}
Let $x \in \R^n$ denote a
random vector with independent entries $x_i$ with $\E[x_i ] = 0 $ and $|x_i | \le K$. Let $A$ be an $n \times n$ matrix. Then, for every $t \ge 0$,
\begin{align*}
    \Pr[|x^\top Ax - \E[x^\top Ax]| > t] \le 2 \cdot \exp(-c \min\{t^2 /(K^4 \|A\|^2_F ), t/(K^2 \|A\|)\}).
\end{align*}

\end{lemma}
\section{Basic Concentration}\label{sec:basic_concentration}

In Section~\ref{sec:basic_concentration:basic_tools}, we introduce some concentration basic tools. In Section~\ref{sec:intialization_and_perturbation:bound_changes_H_w}, given $w$ perturbation within a small ball, we bound the changes of $H$.

\subsection{Some Concentration Basic Tools}\label{sec:basic_concentration:basic_tools}

The goal of this section is to prove Lemma~\ref{lem:bound_on_exp_w_and_perturb_w}.

\begin{lemma}\label{lem:bound_on_exp_w_and_perturb_w}
If the following conditions hold
\begin{itemize}
    \item Let $B > 1$ denote a parameter be defined as Definition~\ref{def:B}.
    \item Let $W = [w_1, \cdots, w_m]$ and $w_r$ be random Gaussian vectors from ${\cal N}(0,\sigma^2 I_d)$. 
    \item Let $V = [v_1, \cdots, v_m]$ and $v_r$ denote the vector where $\| v_r - w_r \|_2 \leq R$, $\forall     r \in [m]$.
    \item Let $x_i\in \R^d$ and $\| x_i \|_2 \leq 1$, $\forall i \in [n]$.
    \item Let $R \in (0,0.01)$.
    \item Let $\S_i$ and $\wt\S_i$ be the softmax function corresponding to $W$ and $V$ respectively.
    \item Let ${\alpha}_i = \langle {\bf 1}_m , \exp({W}^{\top} x_i) \rangle$ and $\wt{\alpha}_i = \langle {\bf 1}_m , \exp(V^{\top} x_i) \rangle$, $\forall i \in [n]$.
\end{itemize}
Then, with probability at least $1-\delta/\poly(nd)$, we have 
\begin{itemize}
    \item Standard inner product
    \begin{itemize}
        \item Part 1. $| \langle w_r, x_i \rangle | \leq B$, $\forall i\in [n]$, $\forall r\in [m]$
        \item Part 2. $|\langle v_r, x_i \rangle | \leq B + R$, $\forall i\in [n]$, $\forall r\in [m]$
        \item Part 3. $| \langle w_r - v_r, x_i + x_j \rangle | \leq 2R$, $\forall i,j \in [n]$, $\forall r\in [m]$
    \end{itemize}
    \item $\exp$ function
    \begin{itemize}
        \item Part 4. $\exp( -B ) \leq \exp(\langle w_r , x_i \rangle) \leq \exp( B )$, $\forall i\in [n]$, $\forall r\in [m]$
        \item Part 5. $\exp( -B-R ) \leq \exp(\langle v_r, x_i \rangle) \leq \exp( B + R )$, $\forall i\in [n]$, $\forall r\in [m]$
        \item Part 6. $|\exp( \langle w_r - v_r, x_i + x_j \rangle ) - 1 | \leq 4R$, $\forall i,j \in [n]$, $\forall r\in [m]$
        \item Part 7. $|\exp(\langle w_r , x_i \rangle) - \exp(\langle v_r , x_i \rangle) | \leq R \exp(B+R)$, $\forall i \in [n]$, $\forall r\in [m]$ 
    \end{itemize}
     {{}{}{} \item softmax $\S$ function
    \begin{itemize}
        \item Part 8. $|\alpha_i - \wt \alpha_i| \le m R \exp(B+R), \forall i \in [n]$ 
        \item  Part 9.  $|\alpha_i^{-1}- \wt\alpha_i^{-1}| \le  \frac{R}{m} \exp(3B+2R), \forall i \in [n]$ 
        \item  Part 10. $ | \S_{i, r} |  \leq \exp(2B)/m, \forall i \in [n], \forall r \in [m]$
        \item  Part 11. $ |\wt \S_{i, r} |  \leq \exp(2B+2R)/m, \forall i \in [n], \forall r \in [m]$
        \item  Part 12. $ | \S_{i, r} - \wt{S}_{i,r} | \leq \frac{R}{m} \exp(4B+3R), \forall i \in [n], \forall r \in [m]$
        \item Part 13.  for any $z\in \R^m$ and $\|z\|_\infty \le 1$, we have $ |\langle z, \S_i \rangle - \langle z, \wt\S_i \rangle| \le R\exp(4B+3R), \forall i \in [n]$
    \end{itemize}
    }
\end{itemize}
\end{lemma}
\begin{proof}
As eventually we choose $m = \poly(nd)$, we use $B > 0$ defined in Definition~\ref{def:B}.

{\bf Proof of Part 1, 2, 4 and 5.}

We can get the proof by Gaussian tail bound.

{\bf Proof of Part 3 and 6.}

Due to $\|x_i\|_2 \leq 1$ and $\|x_j\|_2 \leq 1$ and $\|\Delta w_r\|_2\leq R$, we can have
\begin{align}\label{eq:x_i_x_j}
    |\langle \Delta w_r,(x_i+x_j)\rangle|\leq 2R \leq 0.1.
\end{align}

Then, we have
\begin{align*}
    |\exp( \langle\Delta w_r, (x_i+x_j)\rangle)-1 |
    \leq & ~ 2 |\langle \Delta w_r, (x_i+x_j)\rangle| \notag \\
    \leq & ~ 4 R
\end{align*}
where the first step follows from the Fact~\ref{fac:exp},
and the last step follows from Eq.~\eqref{eq:x_i_x_j}.

{\bf Proof of Part 7.}
Because $\|x_i\|_2 \leq 1$ and $\|\Delta w_r\|_2\leq R$, we can have
\begin{align}\label{eq:x_i}
    |\langle \Delta w_r,x_i\rangle|\leq R \leq 0.1.
\end{align}

By convex increasing property of $\exp$ function, 
we have
\begin{align*}
    |\exp(\langle w_r , x_i \rangle) - \exp(\langle v_r , x_i \rangle) | \leq & \max \{\exp'(\langle w_r , x_i \rangle) , \exp'(\langle v_r , x_i \rangle\}  \cdot | \langle \Delta w_r, x_i\rangle| \\
    \leq & ~ \exp(B+R) \cdot | \langle \Delta w_r, x_i\rangle| \\
    \le & \exp(B+R) R.
\end{align*}
where the first step follows from Taylor expansion and $\exp'$ denote the derivative of $\exp$, the second step follows from Part 4 and Part 5 and the last step follows from Eq.~\eqref{eq:x_i}.

{{}{}{}

{\bf Proof of Part 8.}

\begin{align*}
    |\alpha_i - \wt \alpha_i| = & ~|\sum_{r\in [m]}{\exp}_{i,r} - \wt \sum_{r\in [m]}{\exp}_{i,r} | \\
    \le &  ~ \sum_{r\in [m]} |{\exp}_{i,r} - \wt {\exp}_{i,r} | \\
    \le & ~ m R \exp(B+R),
\end{align*}
where the third step is due to Part 7. 

{\bf Proof of Part 9.}

Similarly, we have
\begin{align*}
    |\alpha_i^{-1}- \wt\alpha_i^{-1}| 
    = & ~ | \frac{ \wt{\alpha}_i - \alpha_i }{ \alpha_i \wt{\alpha}_i } | \\
    \le & ~ \frac{ m R \exp(B+R) }{ | \alpha_i \wt{\alpha}_i |}  \\
    \le & ~ \frac{ m R \exp(B+R) }{ | m \exp(-B)  m \exp(-B-R) |}  \\
    = & \frac{R}{m} \exp(3B+2R). 
\end{align*}
where the first step is due to simple algebra, the second step is from Part 8, the third step follows Part 4, 5, and the last step is because of simple algebra.

{\bf Proof of Part 10 and 11.}

Trivially follows Part 4 and Part 5. 

{\bf Proof of Part 12.}

\begin{align*}
 | \S_{i, r} - \wt{S}_{i,r} | = & ~ |\alpha^{-1}_i\exp_{i,r} - \wt\alpha^{-1}_i \wt \exp_{i,r}| \\
 \le & ~ |\alpha^{-1}_i\exp_{i,r} - \alpha^{-1}_i \wt \exp_{i,r}| + |\alpha^{-1}_i \wt \exp_{i,r} - \wt\alpha^{-1}_i \wt \exp_{i,r}| 
\end{align*}
For the first part, we have
\begin{align*}
    |\alpha^{-1}_i\exp_{i,r} - \alpha^{-1}_i \wt \exp_{i,r}| = & ~ \alpha^{-1}_i |\exp_{i,r} - \wt \exp_{i,r}| \\
    \le & ~  \alpha^{-1}_i \exp(B+R) R\\
    \le & ~  \frac{ \exp(B+R) R}{m \exp(-B)} \\
    = & ~ \frac{R}{m}\exp(2B+R),
\end{align*}
where the second step follows Part 7 and the third step follows Part 4. 

For the second part, we have
\begin{align*}
    |\alpha^{-1}_i \wt \exp_{i,r} - \wt\alpha^{-1}_i \wt \exp_{i,r}|  = & ~ \wt \exp_{i,r} |\alpha^{-1}_i  - \wt\alpha^{-1}_i | \\
    \le & ~  \wt \exp_{i,r} \frac{R}{m} \exp(3B+2R) \\
    \le & ~  \exp(B+R) \frac{R}{m} \exp(3B+2R)\\
    = & ~  \frac{R}{m} \exp(4B+3R),
\end{align*}
where the second step follows Part 9, and the third step follows Part 5.

Thus, we have 
\begin{align*}
 | \S_{i, r} - \wt{S}_{i,r} | \le \frac{R}{m} \exp(4B+3R).
\end{align*}

{\bf Proof of Part 13.}

Note that $\|z\|_\infty \le 1$. We have
\begin{align*}
    |\langle z, \S_i \rangle - \langle z, \wt\S_i \rangle| 
    = & ~ |\langle z, \S_i  - \wt\S_i \rangle| \\
    \le & m \|\S_i  - \wt\S_i\|_\infty \\
    \le & R \exp(4B+3R)  
\end{align*}  
where the first step follows from simple algebra, the second step follows from $|\langle a, b \rangle| \leq m \cdot \max_{i \in [m]} |a_i b_i|$, and the last step is due to Part 12. 
}
\end{proof}

\subsection{Kernel Perturbation}\label{sec:intialization_and_perturbation:bound_changes_H_w}

The purpose of this section is to prove Lemma~\ref{lem:perturb_w_formal}. In the proof, we do not use concentration inequality. Please see Remark~\ref{rem:no_con} for more details. 

\begin{lemma}[Restatement of Lemma~\ref{lem:perturb_w}]\label{lem:perturb_w_formal}
 
If the following conditions hold
\begin{itemize}
    \item Let $B \ge 1$ denote a parameter be defined as Definition~\ref{def:B}.
    \item Let $R \in (0,0.01)$.
    \item Let $x_i \in \R^{d}$ and $\|x_i\|_2\leq 1$ for all $i \in [n]$.
    \item Let $\wt{W} = [\wt{w}_1, \cdots, \wt{w}_m]\in \R^{d\times m}$, where $\wt{w}_1, \cdots, \wt{w}_m$ are are i.i.d. draw from ${\N}(0,\sigma^2 I_d)$. 
    \item Let $W = [w_1, \cdots, w_m] \in \R^{d\times m}$ and satisfy $\| \wt{w}_r - w_r \|_2 \leq R$
    for any $r\in [m]$. 
    \item Let $v_{\ell,r} = a_{\ell,r} \cdot {\bf 1}_m - a_{\ell} \in \R^m$, for any $\ell \in [d]$ and for any $r \in [m]$.  Note that $a_{\ell,r}$ is the $r$-th in $a_{\ell}$.
    \item Let ${\alpha}_i = \langle {\bf 1}_m , \exp({W}^{\top} x_i) \rangle$ and $\wt{\alpha}_i = \langle {\bf 1}_m , \exp(V^{\top} x_i) \rangle$, $\forall i \in [n]$. 
    \item Let $H$ be defined as Definition~\ref{def:H_s}.
\end{itemize}
Then, we have 

\begin{itemize}
\item Part 1. Then with probability at least $1 - \delta/\poly(nd)$, 
\begin{align*}
| [H_{\ell_1,\ell_2}]_{i, j}(W) - [H_{\ell_1,\ell_2}]_{i, j}(\wt{W}) | \leq  {{}{}{} R  \cdot  \exp(10B).}
\end{align*}

\item Part 2.
Then with probability at least $1-\delta$, we have
\begin{align*}
\| H (W) - H(\wt{W}) \|_F \leq {{}{}{} R  n d \cdot  \exp(10B).}
\end{align*}
\end{itemize}
\end{lemma}
\begin{proof}[Proof of Lemma~\ref{lem:perturb_w}]

We define five real numbers $B_1, B_2, B_3, B_4, B_5 \in \R$ as follows,

\begin{align*}
    B_1 := & ~ {\alpha}_i^{-1} {\alpha}_j^{-1} \frac{1}{m} \sum_{r=1}^m \langle v_{\ell_1, r}, \S_i \rangle  \langle v_{\ell_2, r}, \S_j \rangle \exp_{i,r}\exp_{j,r}- {\alpha}_i^{-1} {\alpha}_j^{-1} \frac{1}{m} \sum_{r=1}^m \langle v_{\ell_1, r}, \S_i \rangle  \langle v_{\ell_2, r}, \S_j \rangle \wt\exp_{i,r}\wt\exp_{j,r} \\
    B_2 := & ~ {\alpha}_i^{-1} {\alpha}_j^{-1} \frac{1}{m} \sum_{r=1}^m \langle v_{\ell_1, r}, \S_i \rangle  \langle v_{\ell_2, r}, \S_j \rangle \wt\exp_{i,r}\wt\exp_{j,r} - {\alpha}_i^{-1} {\alpha}_j^{-1} \frac{1}{m} \sum_{r=1}^m \langle v_{\ell_1, r}, \S_i \rangle  \langle v_{\ell_2, r}, \wt\S_j \rangle \wt\exp_{i,r}\wt\exp_{j,r} \\ 
    B_3 := & ~ {\alpha}_i^{-1} {\alpha}_j^{-1} \frac{1}{m} \sum_{r=1}^m \langle v_{\ell_1, r}, \S_i \rangle  \langle v_{\ell_2, r}, \wt\S_j \rangle \wt\exp_{i,r}\wt\exp_{j,r} - {\alpha}_i^{-1} {\alpha}_j^{-1} \frac{1}{m} \sum_{r=1}^m \langle v_{\ell_1, r}, \wt\S_i \rangle  \langle v_{\ell_2, r}, \wt\S_j \rangle \wt\exp_{i,r}\wt\exp_{j,r} \\ 
    B_4 := & ~ {\alpha}_i^{-1} {\alpha}_j^{-1} \frac{1}{m} \sum_{r=1}^m \langle v_{\ell_1, r}, \wt\S_i \rangle  \langle v_{\ell_2, r}, \wt\S_j \rangle \wt\exp_{i,r}\wt\exp_{j,r} - {\alpha}_i^{-1} \wt{\alpha}_j^{-1} \frac{1}{m} \sum_{r=1}^m \langle v_{\ell_1, r}, \wt\S_i \rangle  \langle v_{\ell_2, r}, \wt\S_j \rangle \wt\exp_{i,r}\wt\exp_{j,r} \\ 
    B_5 := & ~ {\alpha}_i^{-1} \wt{\alpha}_j^{-1} \frac{1}{m} \sum_{r=1}^m \langle v_{\ell_1, r}, \wt\S_i \rangle  \langle v_{\ell_2, r}, \wt\S_j \rangle \wt\exp_{i,r}\wt\exp_{j,r} - \wt{\alpha}_i^{-1} \wt{\alpha}_j^{-1} \frac{1}{m} \sum_{r=1}^m \langle v_{\ell_1, r}, \wt\S_i \rangle  \langle v_{\ell_2, r}, \wt\S_j \rangle \wt\exp_{i,r}\wt\exp_{j,r}  
\end{align*}

Thus, we have
\begin{align*}
 | [H_{\ell_1,\ell_2}]_{i, j}(W) - [H_{\ell_1,\ell_2}]_{i, j}(\wt{W}) | {{}{}{}/ {m^2}} 
\leq & ~ |B_1| + | B_2 | + | B_3 | + | B_4 | + | B_5 |.
\end{align*}

{\bf To bound $B_1$}

We rewrite $B_1$ as 
\begin{align*}
    B_1 = {\alpha}_i^{-1} {\alpha}_j^{-1} \frac{1}{m} \sum_{r=1}^m \langle v_{\ell_1, r}, \S_i \rangle  \langle v_{\ell_2, r}, \S_j \rangle ( \exp({w}_r^{\top} (x_i +x_j) ) - \exp(\wt{w}_r^{\top} (x_i +x_j) ) ).
\end{align*}

Recall that $\| v_{\ell_1,r } \|_{\infty} \leq 2$ and $\| \S_i \|_{1} \leq 1$. Thus, $|\langle v_{\ell_1, r}, \S_i \rangle | \leq 2$.

By Fact~\ref{fac:inner_product}, we know that $|\langle v_{\ell_1, r}, \S_i \rangle  \langle v_{\ell_2, r}, \S_j \rangle | \le 2\cdot 2 = 4$. 
By Part 4 of Lemma~\ref{lem:bound_on_exp_w_and_perturb_w}, with probability $1-\delta/\poly(nd)$, 
we know that  $|\alpha_i^{-1} | \leq \frac{1}{m} \exp(B)$. 

We will condition on the above event is holding in the rest of the proof.

By Part 7 of  Lemma~\ref{lem:bound_on_exp_w_and_perturb_w},
\begin{align*}
|  \exp(\wt{w}_r^{\top} (x_i +x_j) ) - \exp(w_r^{\top} (x_i +x_j) )  | \le 2R\exp(2B+2R).
\end{align*}

Finally, we know that
\begin{align*}
    |B_1| \leq \frac{8R}{m^2} \exp(5 B).
\end{align*}

{\bf To bound $B_2$ and $B_3$}

We can rewrite $B_2$ as follows
\begin{align*}
    |B_2| = & ~ | {\alpha}_i^{-1} {\alpha}_j^{-1} \frac{1}{m} \sum_{r=1}^m \langle v_{\ell_1, r}, \S_i \rangle   \wt\exp_{i,r}\wt\exp_{j,r} (\langle v_{\ell_2, r}, \S_j \rangle -  \langle v_{\ell_2, r}, \wt\S_j \rangle) | \\
    \leq & ~ {\alpha}_i^{-1} {\alpha}_j^{-1} \frac{1}{m} \sum_{r=1}^m |\langle v_{\ell_1, r}, \S_i \rangle|   \wt\exp_{i,r}\wt\exp_{j,r}  | (\langle v_{\ell_2, r}, \S_j \rangle -  \langle v_{\ell_2, r}, \wt\S_j \rangle) |.
\end{align*}
Following the similar strategy as $B_1$, by Part 13 of Lemma~\ref{lem:bound_on_exp_w_and_perturb_w}, 
we know that
\begin{align*}
    |B_2| \leq & \frac{1}{m} \exp(B) \cdot \frac{1}{m} \exp(B) \cdot 2 \cdot \exp(B+R)\cdot \exp(B+R) \cdot 4R\exp(4B+3R) \\
    \le & \frac{8R}{m^2}   \exp(9B).
\end{align*}
Similarly, we have 
\begin{align*}
    |B_3| \le & \frac{8R}{m^2}   \exp(9B).
\end{align*}

{\bf To bound $B_4$ and $B_5$}

For the term $B_4$, we can rewrite
\begin{align*}
    |B_4| = & ~ | ({\alpha}_j^{-1} - \wt{\alpha}_j^{-1}) \cdot {\alpha}_i^{-1}  \frac{1}{m} \sum_{r=1}^m \langle v_{\ell_1, r}, \wt\S_i \rangle  \langle v_{\ell_2, r}, \wt\S_j \rangle \wt\exp_{i,r}\wt\exp_{j,r}| \\
    \leq & ~ | {\alpha}_j^{-1} - \wt{\alpha}_j^{-1}| \cdot {\alpha}_i^{-1}  \frac{1}{m} \sum_{r=1}^m | \langle v_{\ell_1, r}, \wt\S_i \rangle  \langle v_{\ell_2, r}, \wt\S_j \rangle | \wt\exp_{i,r}\wt\exp_{j,r}.
\end{align*}

Thus, by Part 9 of Lemma~\ref{lem:bound_on_exp_w_and_perturb_w}, using similar proof strategy as $B_1$ as know 
\begin{align*}
    |B_4| \leq & \frac{R}{m} \exp(3B+2R) \cdot \frac{1}{m} \exp(B) \cdot 2 \cdot 2 \cdot \exp(B+R)\cdot \exp(B+R) \\
    \le & \frac{4R}{m^2}   \exp(7B).
\end{align*}
Similarly, we have 
\begin{align*}
    |B_5| \le & \frac{4R}{m^2}   \exp(7B).
\end{align*}

 \end{proof}

\section{Induction}\label{sec:induction}

In Section~\ref{sec:converge:mainresult}, we provide the proof of our main result. In Section~\ref{app:induction_weights}, we provide an induction lemma for weights part. In Section~\ref{app:induction_loss}, we provide an induction lemma for loss part. In Section~\ref{app:induction_gradient}, we provide an induction lemma for gradient part.

\subsection{Main Result}\label{sec:converge:mainresult}
Our main result is presented as follows.
\begin{theorem}[Main result. Restatement of Theorem~\ref{thm:formal}]\label{thm:formal_formal}
For any $\epsilon, \delta \in (0,0.1)$, if the following conditions hold 
\begin{itemize}
    \item Let  $\lambda=\lambda_{\min}(H^*)>0$
    \item  Let {{}{}{}$ m = \Omega( \lambda^{-2} n^2 d^2 \exp(18B)\log^2(nd/\delta)  )$}
    \item Let {{}{}{} $\eta = 0.1 \lambda / (m n^2 d^2 \exp(16B))  $}
    \item Let $\wh T = \Omega( (m \eta \lambda)^{-1} \log(nd/\epsilon)  ) = {{}{}{}\Omega( \lambda^{-2}n^2 d^2 \exp(16B) \cdot \log(nd/\epsilon) )} $
\end{itemize}
Then, after $\wh T$ iterations,  with probability at least $1-\delta$, we have
\begin{align*}
    \| F(\wh T) - Y \|_F^2 \leq \epsilon.
\end{align*}
\end{theorem}
\begin{proof}[Proof of Theorem~\ref{thm:formal}]

Let $\sigma = 1$.
We have $\| F(0) - Y \|_F^2 \leq nd$ by Lemma~\ref{lem:bound_init_loss}.

Using the choice of $\wh T$, it follows directly from the alternative application of Lemma~\ref{lem:induction_part_2_loss} and Lemma~\ref{lem:induction_part_1_weights}.

Since $\exp(\Theta(B)) = (nd)^{o(1)}$, we can simplify the $nd \exp(\Theta(B)) = (nd)^{1+o(1)}$.
\end{proof}

\subsection{Induction Part 1. For Weights}\label{app:induction_weights}

We provide an induction lemma for weights part.

\begin{lemma}[Induction Part 1. For Weights]\label{lem:induction_part_1_weights}Let $\tau$ be a fixed integer. 

If the below conditions are true 
\begin{itemize}
    \item General Property 1. Let $\lambda = \lambda_{\min} (H^*) > 0$
    \item General Property 2. {{}{}{} $\eta = 0.1 \lambda / (m n^2 d^2 \exp(16B))  $}
    \item General Property 3. Let $D$ be defined as Definition~\ref{def:D}
    \item General Property 4. $D < R = \lambda/(2nd\exp(10B))$
    \item General Property 5. {{}{}{}$ m = \Omega( \lambda^{-2} n^2 d^2 \exp(18B)\log^2(nd/\delta)  )$}
    \item {\bf Weights Property.} $\| w_r(i) - w_r(0)\|_2 \leq  R$ for all $i \in [\tau]$ 
    \item {\bf Loss Property.}  $\| F(i) - Y \|_F^2 \leq \| F(0) - Y \|_F^2\cdot (1-m\eta \lambda/2)^i$,  $\forall i \in [\tau]$ 
    \item {\bf Gradient Property.} $\eta \| \Delta w_r(i) \|_2 \leq 0.01$, $\forall r \in [m]$, $\forall i
    \in [\tau]$
\end{itemize}
Then, for $\tau+1$ and $\forall r\in [m]$, we have 
\begin{align*}
\| w_r(\tau+1) - w_r(0) \|_2 \leq D.
\end{align*}
\end{lemma}

\begin{proof}

We have
\begin{align} \label{eq:upper_bound_etasum}
& ~ \eta \sum_{i=0}^\infty (1-m\eta \lambda/2)^{i/2}\notag\\
\le & ~ \eta \sum_{i=0}^\infty (1-m\eta \lambda/4)^i \notag \\
\leq & ~ \eta \frac{1}{ m\eta \lambda / 4 } \notag \\
\leq & ~ \frac{4}{m\lambda}
\end{align}
where the first step is due to the Fact~\ref{fac:inequality},
the second stepis due to the Fact~\ref{fac:geometric_series}, 
the last step is because of simple algebra.

We use the gradient's norm to measure the weights difference:
{{}{}{}
\begin{align*}
& ~ \|w_r(0)-w_r(\tau+1)\|_2\\
\le & ~\eta \sum_{i=0}^{\tau} \| \Delta w_r(i) \|_2 \\
\le & ~ \eta \sum_{i=0}^{\tau}   \exp(3B)\sqrt{nd}  \cdot \|F(i)-Y \|_F \\
\le & ~ \eta \exp(3B)\sqrt{nd} \sum_{i=0}^{\tau} (1-{m\eta \lambda}/{2})^{i/2} \cdot \|F(0)-Y \|_F \\
\le & ~ 4 m^{-1}\lambda^{-1} \exp(3B) \sqrt{nd} \cdot    \|F(0)-Y \|_F \\
= & ~ D
\end{align*}
}
where the first step follows from $w_r(i+1)-w_r(i)=\eta \cdot \Delta w_r(i)$, the second step follows from Lemma~\ref{lem:bound_Delta_w_at_time_s}  for $\tau$ times, the third step follows from {\bf Loss Property} in Lemma statement, the fourth step follows from Eq.~\eqref{eq:upper_bound_etasum}, the last step is from General Property 3 in Lemma statement.
\end{proof}

\subsection{Induction Part 2. For Loss}\label{app:induction_loss}

We provide an induction lemma for loss part.
 \begin{lemma}[Induction Part 2. For Loss]\label{lem:induction_part_2_loss}
Let $\tau$ be a fixed integer. 

If the following conditions hold
\begin{itemize}
   \item General Property 1. Let $\lambda = \lambda_{\min} (H^*) > 0$
    \item General Property 2. {{}{}{} $\eta = 0.1 \lambda / (m n^2 d^2 \exp(16B))  $}
    \item General Property 3. Let $D$ be defined as Definition~\ref{def:D}
    \item General Property 4. $D < R = \lambda/(2nd\exp(10B))$
    \item General Property 5. {{}{}{}$ m = \Omega( \lambda^{-2} n^2 d^2 \exp(18B)\log^2(nd/\delta)  )$}
    \item {\bf Weights Property.} $\| w_r(\tau) - w_r(0) \|_2 \leq D < R$, $\forall r \in [m]$
    \item {\bf Loss Property.}  $\| F(i) - Y \|_F^2 \leq \| F(0) - Y \|_F^2\cdot (1-m\eta \lambda/2)^i$,  $\forall i \in [\tau]$
    \item {\bf Gradient Property.} $\eta \| \Delta w_r(i) \|_2 \leq 0.01$  $\forall r \in [m]$, $\forall i\in [\tau]$
\end{itemize}
Then we have
\begin{align*}
\| F (\tau+1) - Y \|_F^2 \leq ( 1 - m \eta \lambda / 4 )^{\tau+1} \cdot \| F (0) - Y \|_F^2.
\end{align*}
\end{lemma}
\begin{proof}

We have
\begin{align*}
  \|F(\tau)-Y \|_F^2 
 \leq  \|F(\tau-1)-Y \|_F^2 \cdot (1-m\eta \lambda/2) 
\end{align*}
 which follows Lemma~\ref{lem:loss_one_step_shrinking}. 
 
Thus, we complete the proof by induction.

\end{proof}

\subsection{Induction Part 3. For Gradient}\label{app:induction_gradient}

We provide an induction lemma for gradient part.

\begin{lemma}[Induction Part 3. For Gradient]\label{lem:induction_part_3_gradient}
Let $\tau$ be a fixed integer. 

If the following conditions hold
\begin{itemize}
   \item General Property 1. Let $\lambda = \lambda_{\min} (H^*) > 0$
    \item General Property 2. {{}{}{} $\eta = 0.1 \lambda / (m n^2 d^2 \exp(16B))  $}
    \item General Property 3. Let $D$ be defined as Definition~\ref{def:D}
    \item General Property 4. $D < R = \lambda/(2nd\exp(10B))$
    \item General Property 5. {{}{}{}$ m = \Omega( \lambda^{-2} n^2 d^2 \exp(18B)\log^2(nd/\delta)  )$}
    \item {\bf Weights Property.} $\| w_r(\tau) - w_r(0) \|_2 \leq D < R$, $\forall r \in [m]$
   \item {\bf Loss Property.}  $\| F(i) - Y \|_F^2 \leq \| F(0) - Y \|_F^2\cdot (1-m\eta \lambda/2)^i$,  $\forall i \in [\tau]$
    \item {\bf Gradient Property.} $\eta \| \Delta w_r(i) \|_2 \leq 0.01$  $\forall r \in [m]$, $\forall i \in [\tau]$
\end{itemize}
Then we have
\begin{align*}
\eta \| \Delta w_r(\tau+1) \|_2 \leq 0.01, \forall r \in [m]
\end{align*}
\end{lemma}
\begin{proof}
This is trivially follows from Lemma~\ref{sec:induction:bound_gradient} and Lemma~\ref{lem:bound_Delta_w_times_eta}.

\end{proof}

\section{Induction Part 1: For Weights}\label{sec:induction_for_weight}

In Section~\ref{sec:induction:bound_gradient}, we propose the lemma for bounding gradient and its corresponding proof. In Section~\ref{sec:induction:bound_init_loss}, we propose the bounding initialization loss and its corresponding proof.

\subsection{Bounding the Gradient at any Time}\label{sec:induction:bound_gradient}
In this section, we bound the gradient.

\begin{lemma}\label{lem:bound_Delta_w_at_time_s}
If the following condition hold,
\begin{itemize}
    \item Let $B > 1$ denote a parameter be defined as Definition~\ref{def:B}
    \item Let $R \in (0,0.01)$
    \item $\| w_r(\tau) - w_r(0) \|_2 \leq R$
    \item Let $v_{\ell,r} = a_{\ell,r} \cdot {\bf 1}_m - a_{\ell} \in \R^m$, for any $\ell \in [d]$ and for any $r \in [m]$
\end{itemize}
For any timestamp $\tau$, we have
\begin{align*}
\| \Delta w_r(\tau) \|_2 
\leq {{}{}{} \exp(3B)}\sqrt{nd} \cdot \| F(\tau) - Y \|_F.
\end{align*}
\end{lemma}
\begin{proof}

We have
\begin{align*}
\| \Delta w_r(\tau) \|_2 
= & ~  \left\| {{}{}{}m}\sum_{i=1}^n \sum_{\ell=1}^d (y_{\ell,i} - F_{\ell,i}) \cdot  x_i \cdot  \langle v_{\ell,r} , \S_i(\tau) \rangle \cdot \S_{i,r}(\tau) \right\|_2 \notag\\
\leq & ~  {{}{}{} \exp(3B)}\sum_{i=1}^n \sum_{\ell=1}^d | y_{\ell,i} - F_{\ell,i} (\tau) | \notag\\
\leq & ~  {{}{}{} \exp(3B)}\sqrt{nd} \cdot \| F(\tau) - Y \|_F
\end{align*}
where the first step follows from Claim~\ref{cla:Delta_w_r_at_time_t} and Definition~\ref{def:Delta_w_r_at_time_t}, {{}{}{} the second step follows from 
$| \langle v_{\ell, r} , \S_i \rangle | \leq 2$ and $|\S_{i,r}| \leq \exp(2B+2R)/m$ by Part 11 of Lemma~\ref{lem:bound_on_exp_w_and_perturb_w}, } 
 the last step follows from Cauchy-Schwartz inequality. 
 
\end{proof}

\begin{lemma}\label{lem:bound_Delta_w_times_eta}
If the following conditions hold,
\begin{itemize}
    \item {{}{}{} $\eta = 0.1 \lambda / (m n^2 d^2 \exp(16B))  $}
    \item $\| w_r(\tau) - w_r(0) \|_2 \leq R$
\end{itemize}
Then, for any timestamp $\tau$, we have
\begin{align*}
\eta \| \Delta w_r(\tau) \|_2 
\leq 0.01
\end{align*}
\end{lemma}
\begin{proof}
This trivially follows from Lemma~\ref{sec:induction:bound_gradient} and choice of $\eta$.

\end{proof}

\subsection{Bounding the Initialization Loss}\label{sec:induction:bound_init_loss}

In this section, we bound the initialization loss.

\begin{lemma}\label{lem:bound_init_loss}
We have
\begin{align*}
\| F(0) - Y \|_F \le O(\sqrt{nd}). 
\end{align*}
\end{lemma}
\begin{proof}
This trivially follows from $\|y_i\| \le 1, \forall i \in [n]$ and symmetric initialization from Definition~\ref{def:duplicate_weights}.

\end{proof}
\section{Induction Part 2: For Loss}\label{sec:induction_for_loss}

In Section~\ref{sec:indcution_for_loss:decomposition_v2}, we decompose the loss $\|F(k+1) - Y\|_F^2$ into four parts, namely $C_0, C_1, C_2$, and $C_3$. In Section~\ref{sec:induction_for_loss:choice_parameters}, we show our choices of $m$ and $\eta$. In Section~\ref{sec:induction_for_loss:C_0}, we establish bounds for $C_0$. In Section~\ref{sec:induction_for_loss:C_1}, we establish bounds for $C_1$. 
 In Section~\ref{sec:induction_for_loss:C_2}, we establish bounds for $C_2$.
 In Section~\ref{sec:induction_for_loss:C_3}, we establish bounds for $C_3$.

{{}{}{}
\subsection{Decomposition for \texorpdfstring{$\|\vect(F(\tau+1) - Y)\|_2^2$}{}}\label{sec:indcution_for_loss:decomposition_v2}
Here, we decompose the loss $\|\vect(F(\tau+1) - Y)\|_2^2$ into four parts $C_0, C_1, C_2$ and $C_3$.

\begin{lemma}\label{lem:rewrite_shrinking_one_step_v2}
Assuming the following condition is met:
\begin{itemize}
    \item Let  $\lambda=\lambda_{\min}(H^*)  $
    \item Let $\alpha_i(\tau) := \langle \exp( W(\tau)^\top x_i ), {\bf 1}_m \rangle$.
    \item Let scalar $v_{0,\ell,i} \in \R$ be defined as follows
    \begin{align*}
        v_{0,\ell,i}:= & ~ {{}{}{}m}\sum_{r \in [m]} a_{\ell,r} ( \alpha_i(\tau+1)^{-1} - \alpha_i(\tau)^{-1} ) \cdot ( \exp( \langle w_r(\tau+1),x_i \rangle) )
    \end{align*}
    \item Let scalar $v_{1,\ell,i} \in \R$ be defined as follows
    \begin{align*}
        v_{1,\ell,i}:= & ~ {{}{}{}m}\sum_{r=1}^m a_{\ell,r} \cdot \alpha_i(\tau)^{-1} \exp((\langle w_r(\tau),x_i\rangle) \cdot (-\eta \langle \Delta w_r(\tau), x_i \rangle )
    \end{align*}
    \item Let scalar $v_{2,\ell,i} \in \R$ be defined as follows
    \begin{align*}
        v_{2,\ell,i}:= & ~ {{}{}{}m} \sum_{r=1}^m a_{\ell,r} \cdot \alpha_i(\tau)^{-1} \exp((\langle w_r(\tau),x_i\rangle) \cdot  \eta^2 \cdot \Theta(1)  \cdot  \langle \Delta w_r(\tau), x_i \rangle^2
    \end{align*}
    \item {\bf Gradient Property.} $\eta \| \Delta w_r(i) \|_2 \leq 0.01,$  $\forall r \in [m]$, $\forall i\in [\tau]$
    \item $C_0 = 2 \langle \vect(F(\tau) - Y) , \vect(v_0)\rangle $ 
    \item $C_1 = 2 \langle \vect(F(\tau) - Y) , \vect(v_1)\rangle $ 
    \item $C_2 = 2 \langle \vect(F(\tau) - Y) , \vect(v_2)\rangle $ 
    \item $C_3 = \| F(\tau+1) - F(\tau) \|_F^2$
\end{itemize}
then
\begin{align*}
\|F(\tau+1) - Y \|_F^2 = \| F(t )-Y \|_F^2 + C_0 + C_1 + C_2 + C_3.
\end{align*}
\end{lemma}
\begin{proof}
The expression $\| Y - F(\tau+1) \|_F^2 = \| \vect(Y - F(\tau+1)) \|_2^2$ can be rewritten in the following:
\begin{align}
& ~\| \vect(Y - F(\tau+1)) \|_2^2 \notag \\
= & ~ \| \vect(Y - F(\tau) - ( F(\tau+1) - F(\tau) )) \|_2^2 \notag \\
= & ~ \| \vect(Y - F(\tau)) \|_2^2 - 2 \vect( Y - F(\tau) )^\top  \vect( F(\tau+1) - F(\tau) ) +  \| \vect(F(\tau+1) - F(\tau)) \|_2^2 . \label{eq:induction_split}
\end{align}
where the first step follows from simple algebra, the last step follows from Fact~\ref{fac:squared_euclidean_distance}.

Recall the update rule (Definition~\ref{def:update}),
\begin{align*}
w_{r}(\tau+1) = w_r(\tau) - \eta \cdot \Delta w_{r}(\tau)
\end{align*}

In the following manner, $\forall \ell \in [d]$, we can express $F_{\ell}(\tau+1) - F_{\ell}(\tau) \in \R^n$:
 
\begin{align*}
& ~ F_{\ell,i}(\tau+1) - F_{\ell,i}(\tau) \\
= & ~ {{}{}{}m} \sum_{r \in [m]} a_{\ell,r} \cdot ( \alpha_i(\tau+1)^{-1} \exp( \langle w_r(\tau+1),x_i \rangle) - \alpha_i(\tau)^{-1} \exp(\langle w_r(\tau),x_i \rangle) ) \\
= & ~ +  {{}{}{}m} \sum_{r \in [m]} a_{\ell,r} ( \alpha_i(\tau+1)^{-1} - \alpha_i(\tau)^{-1} ) \cdot ( \exp( \langle w_r(\tau+1),x_i \rangle) ) \\
& ~ + {{}{}{}m} \sum_{r \in [m]} a_{\ell,r} \alpha_i(\tau)^{-1} \cdot ( \exp( \langle w_r(\tau+1),x_i \rangle) - \exp(\langle w_r(\tau),x_i \rangle) )  \\
= & ~ +  {{}{}{}m} \sum_{r \in [m]} a_{\ell,r} ( \alpha_i(\tau+1)^{-1} - \alpha_i(\tau)^{-1} ) \cdot ( \exp( \langle w_r(\tau+1),x_i \rangle) ) \\
& ~ + {{}{}{}m} \sum_{r \in [m]} a_{\ell,r} \cdot \alpha_i(\tau)^{-1} \exp((\langle w_r(\tau),x_i\rangle) \cdot ( \exp(- \eta \langle \Delta w_r(\tau),x_i\rangle) - 1 ) \\
= & ~ +  {{}{}{}m} \sum_{r \in [m]} a_{\ell,r} ( \alpha_i(\tau+1)^{-1} - \alpha_i(\tau)^{-1} ) \cdot ( \exp( \langle w_r(\tau+1),x_i \rangle) ) \\
& ~ + {{}{}{}m} \sum_{r \in [m]} a_{\ell,r} \cdot \alpha_i(\tau)^{-1} \exp((w_r(\tau)^\top x_i) \cdot (-\eta \langle \Delta w_r(\tau), x_i \rangle + \Theta(1) \eta^2 \langle \Delta w_r(\tau), x_i \rangle^2 ) \\
= & ~ v_{0,\ell, i} + v_{1,\ell, i} + v_{2,\ell,i}
\end{align*}
where the first step is due to the definition of $F_{\ell,i}(\tau)$, the second step is from the simple algebra, the third step is due to $|\eta \Delta w_r(\tau)^\top x_i| \leq 0.01$ (due to {\bf Gradient Property} and $\| x_i \|_2 \leq 1$), the fourth step follows from the Fact~\ref{fac:exp_approximation},
the last step follows from
\begin{align*}
v_{0,\ell,i}:= & ~ {{}{}{}m}\sum_{r \in [m]} a_{\ell,r} ( \alpha_i(\tau+1)^{-1} - \alpha_i(\tau)^{-1} ) \cdot ( \exp( \langle w_r(\tau+1),x_i \rangle) ) \\
v_{1,\ell,i}:= & ~ {{}{}{}m}\sum_{r=1}^m a_{\ell,r} \cdot \alpha_i(\tau)^{-1} \exp((\langle w_r(\tau),x_i\rangle) \cdot (-\eta \langle \Delta w_r(\tau), x_i \rangle ) \\
v_{2,\ell,i}:= & ~ {{}{}{}m} \sum_{r=1}^m a_{\ell,r} \cdot \alpha_i(\tau)^{-1} \exp((\langle w_r(\tau),x_i\rangle) \cdot  \eta^2 \cdot \Theta(1)  \cdot  \langle \Delta w_r(\tau), x_i \rangle^2
\end{align*}
Here $v_{0,\ell,i}$ and $v_{1,\ell,i}$ are linear in $\eta$ and $v_{2,\ell,i}$ is quadratic in $\eta$. Thus, $v_{0,\ell,i}$ and $v_{1,\ell,i}$ are the first order term, and $v_{2,\ell,i}$ is the second order term.

We can rewrite the second term in the Eq.~\eqref{eq:induction_split} above as below:
\begin{align*}
 & ~ \langle \vect(Y - F(\tau)), \vect(F(\tau+1) - F(\tau))\rangle \\
= & ~ \langle \vect(Y - F(\tau)) , \vect(v_0 + v_1 + v_2) \rangle \\
= & ~ \langle \vect(Y - F(\tau)) , \vect(v_0)\rangle + \langle \vect(Y - F(\tau)) , \vect(v_1)\rangle + \langle \vect(Y - F(\tau)) , \vect(v_2)\rangle  
\end{align*}

Therefore, we can conclude that
\begin{align*}
\| F(\tau+1)-Y \|_F^2 
=  \| F(\tau)- Y \|_F^2 + C_0 + C_1 + C_2 + C_3.
\end{align*}
\end{proof}
}

\subsection{Choice of Parameters}\label{sec:induction_for_loss:choice_parameters}
Here, we show our choice of parameters $m,\eta, R, B$.
\begin{lemma} \label{lem:loss_one_step_shrinking}
If the below conditions are true
\begin{itemize}
    \item Condition 1. Let $\lambda = \lambda_{\min} (H^*) > 0$ 
    \item Condition 2.  {{}{}{}$ m = \Omega( \lambda^{-2} n^2 d^2 \exp(18B)\log^2(nd/\delta)  )$}
    \item Condition 3. {{}{}{} $\eta = 0.1 \lambda / (m n^2 d^2 \exp(16B))  $}
    \item Condition 4. {{}{}{} $R = \lambda/(2nd\exp(10B))$ } 
    \begin{itemize}
        \item Required by Claim~\ref{cla:C1}
    \end{itemize}
    \item Condition 5. $B = \max\{ C\sigma \sqrt{ \log(nd/\delta) }, ~1 \}$
    \item Condition 6. {{}{}{}$D =4 m^{-1}\lambda^{-1} \exp(3B) \sqrt{nd} \cdot    \|F(0)-Y \|_F$ }
    \item Condition 7. $D < R$
    
    \item Condition 8. $\eta \| \Delta w_r (\tau) \|_2 \leq 0.01$, $\forall r \in [m]$
    \begin{itemize}
        \item Required by Lemma~\ref{lem:rewrite_shrinking_one_step_v2}, Claim~\ref{lem:bound_c0} and Claim~\ref{cla:C3}
    \end{itemize}
\end{itemize}
Then it holds that
\begin{align*}
    \|F(\tau+1)- Y  \|_F^2 \leq  \|  F(\tau) - Y  \|_F^2 \cdot ( 1 - m \eta \lambda / 2 )
\end{align*}
holds with probability at least $1-\delta$.
\end{lemma}
\begin{proof}
We can show
\begin{align*}
& ~ \| F(\tau+1)-Y \|_F^2 \\
= & ~ \|  F(\tau)-Y \|_F^2 + C_0 + C_1 + C_2 + C_3 \\
\leq & ~ {{}{}{} (1- 0.8  m \eta \lambda + 0.1 m \eta \lambda + 2 m \eta^2 n^2 d^2 \exp(9B) + \eta^2 m^2 \cdot n^2 d^2 \cdot \exp(16B) ) \cdot \|  F(\tau)-Y \|_F^2 }\\
\leq & ~ {{}{}{} (1- 0.7  m \eta \lambda + 2\eta^2 m^2 \cdot n^2 d^2 \cdot \exp(16B) ) \cdot \|  F(\tau)-Y \|_F^2 }.
\end{align*}
{{}{}{}
where the first step follows from Lemma~\ref{lem:rewrite_shrinking_one_step_v2}, the second step follows from Lemma~\ref{lem:bound_c0} for $C_0$, Lemma~\ref{lem:bound_c1}, Claim~\ref{cla:C1} for $C_1$, Claim~\ref{cla:C2} for $C_2$ and Claim~\ref{cla:C3} for $C_3$, the last step follows from the simple algebra.
}

\paragraph{Choice of $\eta$.}

Next, we want to choose $\eta$  such that
\begin{equation}\label{eq:choice_of_eta_R}
{{}{}{} (1- 0.7  m \eta \lambda + 2\eta^2 m^2 \cdot n^2 d^2 \cdot \exp(16B) )  \leq (1- m \eta\lambda/2) }.
\end{equation}

Using the choice of $\eta$ in Condition 3
{{}{}{}
\begin{align*}
 2\eta^2 m^2 \cdot n^2 d^2 \cdot \exp(16B)  \leq  0.2 m \eta \lambda
\end{align*}
}
This indicates:
\begin{align}\label{eq:1-eta_lambda_2}
\|  F(\tau+1) -Y\|_F^2 \leq & ~( 1 - m \eta \lambda / 2 )\cdot \|F(\tau) - Y\|_F^2 . 
\end{align}

\paragraph{Lower bound for $m$, over-parametrization size.}

We require the following conditions
{ {}{}{}
\begin{itemize}
    \item  $m \geq \Omega(\lambda^{-2}n^2 d \exp(18B)\log^2(nd/\delta))$ (required by Lemma~\ref{lem:bound_c0}) 
    \item  $m \geq \Omega(\lambda^{-2}n^2 d \exp(12B)\log^2(nd/\delta))$ (required by Lemma~\ref{lem:bound_c1}) 
    \item $ D= 4 m^{-1}\lambda^{-1} \exp(3B) \sqrt{nd} \cdot    \|F(0)-Y \|_F < R = \lambda/(2nd\exp(10B))  \} $ (required by Condition 7.)
\end{itemize}
}
Therefore, by $\| Y - F(0)\|_F = O(\sqrt{nd})$ from Lemma~\ref{lem:bound_init_loss}, it suffices to choose:
{{}{}{}
 \begin{align*}
 m = \Omega( \lambda^{-2} n^2 d^2 \exp(18B)\log^2(nd/\delta)  ).
 \end{align*}
 }
 \end{proof}

\subsection{Bounding \texorpdfstring{$C_0$}{}}\label{sec:induction_for_loss:C_0}
Here, we explain about how to bound $C_0$.
\begin{lemma}\label{lem:bound_c0}
If the following conditions hold
\begin{itemize}
\item Let scalar $v_{0,\ell,i} \in \R$ be defined as follows
{{}{}{}
\begin{align*}
        v_{0,\ell,i}:= & ~ {{}{}{}m}\sum_{r \in [m]} a_{\ell,r} ( \alpha_i(\tau+1)^{-1} - \alpha_i(\tau)^{-1} ) \cdot ( \exp( \langle w_r(\tau+1),x_i \rangle) )
    \end{align*}
    }
\item Let $\alpha_i(\tau) := \langle \exp( W(\tau)^\top x_i ), {\bf 1}_m \rangle$.
\item Let {{}{}{} $m \ge \Omega(\lambda^{-2} n^2 d \exp(18B)\log^2(nd/\delta) )$}
\item {\bf Gradient Property.} $\eta \| \Delta w_r(i) \|_2 \leq 0.01,$  $\forall r \in [m]$, $\forall i\in [\tau]$
\item We define $C_0$ as follows
\begin{align*}
    C_0 = 2 \langle \vect(F(\tau) - Y) , \vect(v_0)\rangle
\end{align*} 
Here $\vect(v_0) \in \R^{nd}$ is the vectorization of $v_0 \in \R^{n \times d}$ and $\vect(F(\tau) - Y) \in \R^{nd}$ is the vectorization of $F(\tau) - Y \in \R^{n \times d}$.
\end{itemize}
Then we have  
\begin{align*}
    | C_0 | \leq 0.1 m \eta \lambda \cdot \| F(\tau) -Y \|_F^2
\end{align*}
\end{lemma}
\begin{proof}

We can rewrite $v_{0,\ell,i}$ as follows:
\begin{align}\label{eq:v_0_ell_i}
v_{0,\ell,i} 
= & ~ {{}{}{}m}\sum_{r=1}^m a_{\ell,r} ( (\alpha_i(\tau+1))^{-1} - \alpha_i(\tau)^{-1} ) \exp( \langle w_{r}(\tau+1) , x_i \rangle ) \notag \\
= & ~ {{}{}{}m}\sum_{r=1}^m a_{\ell,r} \alpha_i(\tau+1)^{-1} \alpha_i(\tau)^{-1} \cdot ( \langle {\bf 1}_m , \exp(W(\tau+1) x_i ) - \exp(W(\tau) x_i) \rangle ) \exp(\langle w_r(\tau+1), x_i \rangle) \notag \\
= & ~{{}{}{}m}\sum_{r=1}^m a_{\ell,r} \alpha_i(\tau+1)^{-1} \alpha_i(\tau)^{-1} ( \sum_{r_2=1}^m \exp(w_{r_2} (\tau+1)x_i) - \exp(w_{r_2} (\tau) x_i) ) \exp( \langle w_{r} (\tau+1), x_i \rangle ) \notag \\
= & ~{{}{}{}m}(\sum_{r=1}^m a_{\ell,r} \alpha_i(\tau+1)^{-1} \alpha_i(\tau)^{-1} \sum_{r_2=1}^m -\eta \langle \Delta w_{r_2}(\tau) , x_i \rangle \exp(w_{r_2} (\tau) x_i) \exp( \langle w_{r} (\tau+1), x_i \rangle ) + \eta^2 \Delta_2 ) \notag \\
= & ~ {{}{}{}m}(\underbrace{ \sum_{r=1}^m a_{\ell,r}  \sum_{r_2=1}^m -\eta \langle \Delta w_{r_2}(\tau) , x_i \rangle \S_{i, r_2} (\tau) \cdot \S_{i, r} (\tau + 1) }_{\mathrm{first~order~term}}  + \underbrace{ \eta^2 \Delta_2 }_{\mathrm{second~order~term}})
\end{align}
where the first step follows from lemma statement, the second step follows from $a^{-1} -b^{-1}= \frac{b-a}{ab}$, the third step follows from simple algebra, the fourth step follows from simple algebra, and the last step follows from $|\eta \Delta w_r(\tau)^\top x_i| \leq 0.01$ (due to {\bf Gradient Property} and $\| x_i \|_2 \leq 1$).

The second order term $\eta^2 \Delta_2$ in Eq.~\eqref{eq:v_0_ell_i} can be bounded in a similar way as the proof of Claim~\ref{cla:C2}.

Further, we can rewrite the first-order term in Eq.~\eqref{eq:v_0_ell_i} 
\begin{align}
   {{}{}{}m} \sum_{r=1}^m a_{\ell,r}  \sum_{r_2=1}^m -\eta \langle \Delta w_{r_2}(\tau) , x_i \rangle \S_{i, r_2} (\tau) \cdot \S_{i, r} (\tau + 1)
    = {{}{}{}m^2} (Q_{1,i,\ell} + Q_{2,i,\ell})
\end{align}
where 
\begin{align*}
Q_{1,i,\ell} := & ~ {{}{}{} \frac{1}{m} } \sum_{r=1}^m a_{\ell,r}  (  -\eta \langle \Delta w_{r}(\tau) , x_i \rangle) \S_{i,r}(\tau) \cdot \S_{i,r}(\tau+1) \\
Q_{2,i,\ell} := & ~ {{}{}{} \frac{1}{m} } \sum_{r=1}^m a_{\ell,r}   \sum_{r_2\neq r} (-\eta \langle \Delta w_{r_2}(\tau) , x_i \rangle ) \S_{i,r_2}(\tau) \cdot \S_{i,r}(\tau+1)
\end{align*}

Let us consider how to handle the first term in Eq.~\eqref{eq:v_0_ell_i},
\begin{align*}
   Q_{1,i,\ell} = & {{}{}{} \frac{1}{m} }\sum_{r=1}^m a_{\ell,r}  (  -\eta \langle \Delta w_{r}(\tau) , x_i \rangle) \S_{i,r}(\tau) \cdot \S_{i,r}(\tau+1) \notag \\
   = & ~ \sum_{r=1}^m a_{\ell,r}  \S_{i,r} \cdot \S_{i,r}(\tau+1)  ( - \eta   \sum_{j=1}^n \sum_{\ell_2=1}^d  ( F_{\ell_2, j}(\tau) - y_{\ell_2, j} ) \cdot \Big(  (  \langle a_{\ell_2,r} \cdot {\bf 1}_m -  a_{\ell_2}, \S_j \rangle ) \cdot \S_{j,r} \Big) \cdot x_j^\top ) x_i 
\end{align*}
where the second step follows from computing $\Delta w_r(\tau)$ explicitly (see Claim~\ref{cla:Delta_w_r_at_time_t}).

{{}{}{}
Similarly as proof of Lemma~\ref{lem:bound_c1}, we can use concentration to bound 
\begin{align*}
\sum_{i=1}^n \sum_{\ell=1}^d Q_{1,i,\ell} (F_{\ell,i}  - y_{\ell,i})
\end{align*}
Note that $0<\S_{j,r} < \frac{\exp(3B)}{m}$ by Part 11 of Lemma~\ref{lem:bound_on_exp_w_and_perturb_w}. 
The above small term is equivalent to
\begin{align*}
    - \eta \frac{\exp(9B)}{m^3} \cdot \sum_{i=1}^n \sum_{j=1}^n   \sum_{r=1}^m  \sum_{\ell=1}^d \sum_{\ell_2=1}^d ( F_{\ell_2, j}(\tau) - y_{\ell_2, j} ) \cdot \sigma_{i,j,r,\ell,\ell_2} \cdot C_{i,j,r,\ell,\ell_2} \cdot ( F_{\ell, i}(\tau) - y_{\ell, i} ),
\end{align*}
where $\sigma_{i,\ell,\ell_2,j,r} \sim [-1,+1]$ and $|C_{i,\ell,\ell_2,j,r}| \leq 10$.
We define
\begin{align*}
P_{1,r, \ell, \ell_2} := & ~ (F_{\ell_2,j} - y_{\ell_2,j}) \sigma_{i,j,r,\ell,\ell_2}  C_{i,j,r,\ell,\ell_2} (F_{\ell,i} - y_{\ell,i}) 
\end{align*}

Similarly as Lemma~\ref{lem:bound_c1}, for each fixed $i,j \in [n]$, using Hanson-Wright inequality (Lemma~\ref{lem:hanson}), we can show 
\begin{align*}
    & ~ \Pr[ | \sum_{r=1}^m \sum_{\ell=1}^d \sum_{\ell_2=1}^d P_{1,r,\ell,\ell_2} | \leq 100 \| F_{j} - y_j \|_2 \| F_i - y_i \|_2 \cdot \sqrt{md} \log(nd/\delta)  ] \\
\geq & ~ 1- \delta/\poly(nd) .
\end{align*}

By mean inequality, we have
\begin{align*}
\sum_{i=1}^n \sum_{j=1}^n \| F_j - y_j \|_2 \cdot \| F_i - y_i \|_2 
\le n \| F - y \|_F^2.
\end{align*}

Thus, we have the first term with probability at least $1-\poly(nd)$, such that 
\begin{align*}
|\sum_{i=1}^n \sum_{\ell=1}^d Q_{1,i,\ell} (F_{\ell,i}  - y_{\ell,i})| \le \eta \frac{n\exp(9B)}{m^3}  \| F - y \|_F^2 \sqrt{md} \log(nd/\delta)
\end{align*}

Similarly, we can compute 
\begin{align*}
    \sum_{i=1}^n \sum_{\ell=1}^d Q_{2,i,\ell} (F_{\ell,i}  - y_{\ell,i})
\end{align*}
Using Hanson-Wright inequality (Lemma~\ref{lem:hanson}), we have the second term with probability at least $1-\poly(nd)$, such that 
\begin{align*}
|\sum_{i=1}^n \sum_{\ell=1}^d Q_{2,i,\ell} (F_{\ell,i}  - y_{\ell,i})| \le \eta \frac{n\exp(9B)}{m^2}  \| F - y \|_F^2 \sqrt{md} \log(nd/\delta)
\end{align*}

Thus, we can complete the proof by the Lemma statement $m \ge \Omega(\lambda^{-2} n^2 d \exp(18B)\log^2(nd/\delta) )$.
}
\end{proof}

\subsection{Bounding \texorpdfstring{$C_1$}{}}\label{sec:induction_for_loss:C_1}
Here, we give the bound of the first order term $C_1$.
{{}{}{}
Note that this term is making progress. 
\begin{lemma}\label{lem:bound_c1}
Assuming the following condition is met:
\begin{itemize}
    \item Let  $\lambda=\lambda_{\min}(H^*)  $
    \item Let $\alpha_i(\tau) := \langle \exp( W(\tau)^\top x_i ), {\bf 1}_m \rangle$
    \item Let $m \geq \Omega(\lambda^{-2}n^2 d \exp(12B)\log^2(nd/\delta))$
    \item Let scalar $v_{1,\ell,i} \in \R$ be defined as follows
    \begin{align*}
        v_{1,\ell,i}:= & ~ {{}{}{}m}\sum_{r=1}^m a_{\ell,r} \cdot \alpha_i(\tau)^{-1} \exp((\langle w_r(\tau),x_i\rangle) \cdot (-\eta \langle \Delta w_r(\tau), x_i \rangle )
    \end{align*}
    \item $C_1 = 2 \langle \vect(F(\tau) - Y) , \vect(v_1)\rangle $ 
\end{itemize}
then
\begin{align*}
C_1  \leq -1.6 m \eta \vect(F(\tau) - Y)^\top H(\tau) \vect(F(\tau) - Y).
\end{align*}
\end{lemma}
\begin{proof}
To simplify the notation, we omit writing $(\tau)$ in $\S_{i,r}(\tau)$. Then, we can express $v_{1, \ell,i} \in \R$ as follows: 
\begin{align}\label{eq:v_1_ell_i}
v_{1,\ell,i} 
= & ~  m \sum_{r \in [m]} a_{\ell,r}  \cdot \S_{i,r} \cdot (-\eta \langle x_i,\Delta w_r(\tau) \rangle ) \notag \\
= & ~ m^2 \sum_{r \in [m]} a_{\ell,r}  \cdot \S_{i,r} \cdot ( - \eta   \sum_{j=1}^n \sum_{\ell_2=1}^d  ( F_{\ell_2, j}(\tau) - y_{\ell_2, j} ) \cdot \Big(  (  \langle a_{\ell_2,r} \cdot {\bf 1}_m -  a_{\ell_2}, \S_j \rangle ) \cdot \S_{j,r} \Big) \cdot x_j^\top ) x_i \notag \\
= & ~ m^2( Q_{1,\ell,i} + Q_{2,\ell,i})
\end{align}
where the second step using equation for $\Delta w_r(\tau) $ (see Claim~\ref{cla:Delta_w_r_at_time_t}).

Note that $\langle a_{\ell,r} \cdot {\bf 1}_m, S_{i} \rangle =  a_{\ell,r}$, so in the above equation,
\begin{align*}
Q_{1,\ell,i} := & ~ \sum_{r \in [m]} \langle a_{\ell,r} \cdot {\bf 1}_m - a_{\ell}, S_{i} \rangle   \cdot \S_{i,r} \cdot ( - \eta   \sum_{j=1}^n \sum_{\ell_2=1}^d  ( F_{\ell_2, j}(\tau) - y_{\ell_2, j} ) \cdot \Big(  (  \langle a_{\ell_2,r} \cdot {\bf 1}_m -  a_{\ell_2}, \S_j \rangle ) \cdot \S_{j,r} \Big) \cdot x_j^\top ) x_i\\
 Q_{2,\ell,i}: = & ~ \sum_{r \in [m]} \langle a_{\ell}, S_{i} \rangle   \cdot \S_{i,r} \cdot ( - \eta   \sum_{j=1}^n \sum_{\ell_2=1}^d  ( F_{\ell_2, j}(\tau) - y_{\ell_2, j} ) \cdot \Big(  (  \langle a_{\ell_2,r} \cdot {\bf 1}_m -  a_{\ell_2}, \S_j \rangle ) \cdot \S_{j,r} \Big) \cdot x_j^\top ) x_i 
\end{align*}
The quantity $\sum_{i\in[n]}\sum_{\ell \in [d]} Q_{1,\ell,i} (F_{\ell,i}-Y_{\ell,i})$ is corresponding to first term $(Q_{1,\ell,i})$ in Eq.~\eqref{eq:v_1_ell_i}. It is 
\begin{align}\label{eq:v_1_ell_i_Q_1}
\sum_{i\in[n]}\sum_{\ell \in [d]} Q_{1,\ell,i} (F_{\ell,i}-Y_{\ell,i})
= - \frac{1}{m} \eta \vect(F(\tau) - Y )^\top H(\tau)^\top \vect(F(\tau) - Y )
\end{align}

The quantity $\sum_{i\in[n]}\sum_{\ell \in [d]} Q_{2,\ell,i} (F_{\ell,i}-Y_{\ell,i})$ is corresponding to second term $(Q_{2,\ell,i})$ in Eq.~\eqref{eq:v_1_ell_i}. Note that $0<\S_{j,r} < \frac{\exp(3B)}{m}$ by Part 11 of Lemma~\ref{lem:bound_on_exp_w_and_perturb_w}. The quantity,
\begin{align}\label{eq:v_1_ell_i_Q_2}
\sum_{i\in[n]}\sum_{\ell \in [d]} Q_{2,\ell,i} (F_{\ell,i}-Y_{\ell,i})
\end{align}
is equivalent to 
\begin{align*}
    - \eta \frac{\exp(6B)}{m^2} \cdot \sum_{i=1}^n \sum_{j=1}^n   \sum_{r=1}^m  \sum_{\ell=1}^d \sum_{\ell_2=1}^d ( F_{\ell_2, j}(\tau) - y_{\ell_2, j} ) \cdot \sigma_{i,j,r,\ell,\ell_2} \cdot C_{i,j,r,\ell,\ell_2} \cdot ( F_{\ell, i}(\tau) - y_{\ell, i} ),
\end{align*}
where $\sigma_{i,j,r,\ell,\ell_2} \in \{-1,+1\}$ and $|C_{i,j,r,\ell,\ell_2}| \leq 10$.

Note that there are four cases
\begin{itemize}
 \item $i = j $, $\ell= \ell_2$, this is a p.s.d. 
 case that always makes progress, thus we can drop it.
\item $i \neq j $, $\ell= \ell_2$ we will use random variable $P_1$ to handle
\item $i = j $, $\ell \neq \ell_2$ we will use random variable $P_2$ to handle
\item $i \neq j $, $\ell \neq \ell_2$ we will use random variable $P_2$ to handle
\end{itemize}

For each fixed $i,j \in [n]$.
We define
\begin{align*}
 P_{1, r, \ell} := & ~  (F_{\ell,j} - y_{\ell,j}) \sigma_{i,j,r,\ell}  C_{i,j,r,\ell} (F_{\ell,i} - y_{\ell,i}) \\
P_{2,r, \ell, \ell_2} := & ~ (F_{\ell_2,j} - y_{\ell_2,j}) \sigma_{i,j,r,\ell,\ell_2}  C_{i,j,r,\ell,\ell_2} (F_{\ell,i} - y_{\ell,i}) 
\end{align*}
The random variables related to $P_{1,r,\ell}$ are the following
\begin{align*}
\sum_{r=1}^m \sum_{\ell=1}^d P_{1,r,\ell}
\end{align*}

The random variables related to $P_{2,r,\ell,\ell_2}$ are the following
\begin{align*}
\sum_{r=1}^m \sum_{\ell=1}^d \sum_{\ell_2=1}^d P_{2,r,\ell,\ell_2}
\end{align*}

For each $i\neq j \in [n]$ and $\ell= \ell_2$, using Hoeffding inequality (see Lemma~\ref{lem:hoeffding}),
we can show
\begin{align*}
    & ~ \Pr[ | \sum_{r=1}^m \sum_{\ell=1}^d P_{1,r,\ell} | \leq 100 \| F_{j} - y_j \|_2 \| F_i - y_i \|_2 \cdot \sqrt{ md  \log(nd/\delta)  } ] \\
\geq & ~ 1- \delta/\poly(nd) .
\end{align*}

Similarly, we consider $i= j$ and $\ell \neq \ell_2$ by Hanson-Wright inequality (Lemma~\ref{lem:hanson}), we have
\begin{align*}
    & ~ \Pr[ | \sum_{r=1}^m \sum_{\ell=1}^d \sum_{\ell_2=1}^d P_{2,r,\ell,\ell_2} | \leq 100 \| F_{j} - y_j \|_2 \| F_i - y_i \|_2 \cdot \sqrt{md} \log(nd/\delta)  ] \\
\geq & ~ 1- \delta/\poly(nd) .
\end{align*}

By mean inequality, we have
\begin{align*}
\sum_{i=1}^n \sum_{j=1}^n \| F_j - y_j \|_2 \cdot \| F_i - y_i \|_2 
\le n \| F - y \|_F^2.
\end{align*}

Note that by Lemma condition, we have 
\begin{align*}
\frac{1}{m} \lambda \gtrsim \frac{n\exp(6B)}{m^2} \cdot  \sqrt{ md } \log(nd/\delta)  \iff m \gtrsim \lambda^{-2},
\end{align*}
the equation (Eq.~\eqref{eq:v_1_ell_i_Q_1} and the bound for Eq.~\eqref{eq:v_1_ell_i_Q_2}) above indicates that $\langle \vect(Y - F(\tau)) , \vect(v_1)\rangle$ can be expressed as  
\begin{align}\label{eq:rewrite_v1}
\vect(v_{1})^\top \vect( Y - F(\tau) ) \ge 0.8 m \eta \cdot \underbrace{ \vect( F(\tau) - Y  )^\top }_{1 \times nd} \underbrace{ H( \tau )^\top }_{nd \times nd} \vect( F(\tau) - Y  ).
\end{align}
We finish the proof.
\end{proof}
}

\begin{claim}\label{cla:C1}
If the below conditions are true
\begin{itemize}
    \item Let $B \ge 1$ be defined as Definition~\ref{def:B}
    \item Let $\lambda = \lambda_{\min} (H^*) > 0$
    \item $C_1 = -m \eta \vect(F(\tau) - Y)^\top H(\tau) \vect(F(\tau) - Y).$
    \item {{}{}{} $R = \lambda/(2nd\exp(10B))$} 
\end{itemize}
Then, we have
\begin{align*}
C_1 \leq - \frac{1}{2} m \eta \lambda\cdot \|  F(\tau) 
 - Y \|_F^2 
\end{align*}
and 
\begin{align*}
\lambda_{\min}(H(\tau)) \ge \lambda/2.
\end{align*}
holds with probability at least $1-\delta$.
\end{claim}

\begin{proof}
By Lemma \ref{lem:perturb_w}, with probability at least $1-\delta$, we have
{{}{}{} 
\begin{align}\label{eq:upper_lambda}
   & ~ \|H^*-H(\tau)\|_F \notag \\
   \leq & ~ R n d \cdot  \exp(10B)\notag\\
   \leq & ~ \lambda / 2
\end{align}
}
where the first step follows from the definition of $H(\tau)$, the last step comes from choice of $\lambda$ (see Claim Statement).

Given that $\lambda=\lambda_{\min}(H^*)$, by eigenvalue perturbation theory  
\begin{align*}
& ~ \lambda_{\min}(H(\tau)) \\
\geq & ~ \lambda_{\min}(H^*)- \|H^*-H(\tau)\|\\
\geq & ~ \lambda_{\min}(H^*)- \|H^*-H(\tau)\|_F \\
\geq & ~ \lambda_{\min}(H^*) - \lambda/2 \\
\geq & ~ \lambda /2.
\end{align*}
where the first step comes from triangle inequality, the second step is due to Frobenius norm, the third step is due to   Eq.\eqref{eq:upper_lambda}, the last step follows from $\lambda_{\min}(H^*) = \lambda $. 

Finally, we have
\begin{align*}
  \vect(F(\tau) - Y)^\top H(\tau) \vect(F(\tau) - Y ) \geq \lambda / 2\cdot\|F(\tau)-Y \|_F^2 .
\end{align*}
Thus, we complete the proof.
\end{proof}

{{}{}{}
\subsection{Bounding \texorpdfstring{$C_2$}{}}\label{sec:induction_for_loss:C_2}
Here, we give the bound of the second order term $C_2$.

\begin{claim}\label{cla:C2}
If the below conditions are true
\begin{itemize}
    \item Let  $\lambda=\lambda_{\min}(H^*)  $
    \item Let $\alpha_i(\tau) := \langle \exp( W(\tau)^\top x_i ), {\bf 1}_m \rangle$
    \item Let scalar $v_{2,\ell,i} \in \R$ be defined as follows
    \begin{align*}
        v_{2,\ell,i}:= & ~ {{}{}{}m} \sum_{r=1}^m a_{\ell,r} \cdot \alpha_i(\tau)^{-1} \exp((\langle w_r(\tau),x_i\rangle) \cdot  \eta^2 \cdot \Theta(1)  \cdot  \langle \Delta w_r(\tau), x_i \rangle^2
    \end{align*}
    \item $C_2 = 2 \langle \vect(F(\tau) - Y) , \vect(v_2)\rangle $ 
\end{itemize}
Then we can conclude that
\begin{align*}
C_2 \leq 2 m \eta^2 n^2 d^2 \exp(9B) \|F(\tau) - Y\|_F^2.
\end{align*}
with probability at least $1-n\cdot \exp(-mR)$.
\end{claim}
\begin{proof}
Let $p_{i,r} \in [-1,1]$. We have
\begin{align*}
|v_{2,\ell, i} |
= & ~ m \sum_{r\in [m ]} a_{\ell,r}  \cdot \S_{i,r} \cdot (\eta^2 p_{i,r} \langle  x_i, \Delta w_r(\tau) \rangle^2 ) \\
\le & ~ m \eta^2 n d \exp(9B) \|F(\tau) - Y\|_F^2,
\end{align*}
where the last step follows Lemma~\ref{lem:bound_Delta_w_at_time_s} and Part 11 of Lemma~\ref{lem:bound_on_exp_w_and_perturb_w}.

Thus, 
\begin{align}
    C_2 = & ~ 2 \langle \vect(F(\tau) - Y) , \vect(v_2)\rangle \notag \\
    \le & ~ 2 \|F(\tau) - Y\|_F \|v_2\|_F \notag \\
    \le & ~ 2 m \eta^2 n^2 d^2 \exp(9B) \|F(\tau) - Y\|_F^2, \notag
\end{align}
where the first step follows Cauchy-Schwartz
inequality, and the second step follows $\|F(\tau) - Y\|_F \le O(\sqrt{nd})$ by induction statement (See Lemma~\ref{lem:induction_part_2_loss}). 

\end{proof}
}

\subsection{Bounding \texorpdfstring{$\| F(\tau+1) - F(\tau) \|_F^2$}{}}\label{sec:induction_for_loss:C_3}
Here, we give the bound of the third order term $C_3$.

\begin{claim}\label{cla:C3}
If the below conditions are true
\begin{itemize}
\item Let $B \ge 1$ be defined as Definition~\ref{def:B}
 \item $C_{3}  = \| F(\tau+1) - F(\tau) \|_F^2$.
 \item $R \in (0,0.01)$
 \item {\bf Gradient Property.} $\eta \| \Delta w_r(i) \|_2 \leq 0.01,$  $\forall r \in [m]$, $\forall i\in [\tau]$
\end{itemize}
Then with probability at least $1-\delta$, we have
\begin{align*}
C_3 \leq {{}{}{} \eta^2 m^2 \cdot n^2 d^2 \cdot \exp(16B) \cdot \| F(\tau)-Y \|_F^2 }.
\end{align*}
\end{claim}

\begin{proof}
Note that we denote ${\alpha}_i$ as $\langle {\bf 1}_m , \exp({W}^{\top} x_i) \rangle$.
According to definition of $F_{\ell, i}(\tau)$, we have
\begin{align*}
& ~ F_{\ell, i}(\tau+1) - F_{\ell, i}(\tau) \\
= & ~  {{}{}{} m} a_{\ell}^\top (  \\
& ~ +\alpha_i(\tau+1)^{-1} \exp((W(\tau+1)^\top x_i)   -  \alpha_i(\tau)^{-1} \exp((W(\tau+1)^\top x_i)  \\
& ~ +  \alpha_i(\tau)^{-1} \exp((W(\tau+1)^\top x_i) -  \alpha_i(\tau)^{-1} \exp((W(\tau)^\top x_i) \\
& ~ ) 
\end{align*}

Then we have
\begin{align}\label{eq:bound_u_i_k}
& |F_{\ell, i}(\tau+1) - F_{\ell, i}(\tau) | \\
\leq & ~ {{}{}{} m} \sum_{r=1}^m |\alpha_i(\tau+1)^{-1} - \alpha_i(\tau)^{-1}| \exp(w_r(\tau+1)^\top x_i) \notag \\
& + {{}{}{} m} \sum_{r=1}^m  \alpha_i(\tau)^{-1} \exp(w_r(\tau)^\top x_i) \cdot | \exp(- \eta \Delta w_r(\tau)^\top x_i) - 1 | \notag 
\end{align}
where it follows from triangle inequality.

For the second term in Eq.~\eqref{eq:bound_u_i_k}, we have  
\begin{align*}
& {{}{}{} m} \sum_{r=1}^m  \alpha_i(\tau)^{-1} \exp(w_r(\tau)^\top x_i) \cdot | \exp(- \eta \Delta w_r(\tau)^\top x_i) - 1 | \\
\leq & ~ \exp(B+R) \exp(B+R) \sum_{r=1}^m   | \exp(- \eta \Delta w_r(\tau)^\top x_i) - 1 | \notag \\
\leq & ~  \exp(2B+2R) \sum_{r=1}^m  2 \eta \| \Delta w_r(\tau) \|_2 \notag \\
= & ~ 2 \eta \exp(2B+2R) \sum_{r=1}^m \| \Delta w_r(\tau) \|_2 \notag \\
\leq & ~ 2 \eta \exp(2B+2R) \cdot m \cdot {{}{}{} \exp(3B)} \sqrt{nd}  \| F(\tau) - Y \|_F\notag \\
\le & ~  \eta {{}{}{} m}{{}{}{} \exp(6B)} \sqrt{nd} \| F(\tau) - Y \|_F
\end{align*}     
where the first step comes from Lemma~\ref{lem:bound_on_exp_w_and_perturb_w}, the second step is due to $\eta \| \Delta w_r(\tau) \|_2 \leq 0.01 $ (this is stated in Claim assumption) and Fact~\ref{fac:exp_approximation}, the third step is from simple algebra, the fourth step is due to Lemma~\ref{lem:bound_Delta_w_at_time_s}, the last step follows from simple algebra. 

Similarly, for the first term in Eq.~\eqref{eq:bound_u_i_k} we have
\begin{align*}
& {{}{}{} m} \sum_{r=1}^m |\alpha_i(\tau+1)^{-1} - \alpha_i(\tau)^{-1}| \exp(w_r(\tau+1)^\top x_i) \\ 
\le & ~ {{}{}{} m^2} \exp(B+R) |\alpha_i(\tau+1)^{-1} - \alpha_i(\tau)^{-1}| \\ 
\le & ~ {{}{}{}  m \exp(B+R) |\eta \Delta w_r(\tau)^\top x_i | \exp(3B+2R)} \\
\le & ~ {{}{}{}  \eta m \exp(4B+3R)  \|\Delta w_r(\tau)\|_2 } \\
\le & ~ {{}{}{}  \eta m \exp(7B+3R) \sqrt{nd}  \| F(\tau) - Y \|_F  } 
\end{align*}
where the first step follows from Part 5 of Lemma~\ref{lem:bound_on_exp_w_and_perturb_w}, 
{{}{}{}
the second step follows from Part 9 of Lemma~\ref{lem:bound_on_exp_w_and_perturb_w} where $R = |\eta \Delta w_r(\tau)^\top x_i|$, the third step follows from simple algebra, 
and the last step follows from Lemma~\ref{lem:bound_Delta_w_at_time_s}. }

Thus we have
\begin{align}\label{eq:bound_u_i_k_end}
|F_{\ell, i}(\tau+1) - F_{\ell, i}(\tau) | 
\leq & ~ {{}{}{} \eta m\exp(8B)  \sqrt{nd}  \| F(\tau) - Y \|_F} .
\end{align}

Finally, we get
\begin{align*}
\| F(\tau+1) - F(\tau) \|_F^2 \leq & ~ nd \cdot ( {{}{}{} \eta m\exp(8B)  \sqrt{nd}  \| F(\tau) - Y \|_F} )^2 \\
\le & ~ {{}{}{} \eta^2 m^2 \cdot n^2 d^2 \cdot \exp(16B) \cdot \| F(\tau)-Y \|_F^2 }
\end{align*}
where the first step is because of Eq.~\eqref{eq:bound_u_i_k_end}, 
the last step comes from simple algebra. 
\end{proof}

\section{NTK Regression}\label{sec:ntk_regression}

In this section, we introduce the NTK regression, as we will show that the neural network is ``equivalent'' to this regression so that we can give a final guarantee on the test data. 
To clarify the function, we use $F_{nn}$ to denote $F$ as a neural network function. We use $x_{te} \in \R^d$ to denote the test data.  We would like to control the error between the neural network $F_{nn}$ and the function $F_{ntk}$.  For convenience, we call this error ``coupling error'', which is the difference between the trained neural network and its corresponding NTK regression. 

Recall that, by Definition~\ref{def:H_s}, we have the $H^* = H(W(0))$. Recall  $[H^*]_{i,j} \in \R^{d\times d}$ is the kernel between $x_i$ and $x_j$. Similarly, $\forall \ell_1,\ell_2 \in [d]$, for test data, we can define the NTK induced feature map as
\begin{align*}
[K^*_{\ell_1,\ell_2}]_{te, j}:= & \frac{1}{m} x_{te}^{\top} x_{j} \sum_{r=1}^{m} \langle v_{\ell_1, r}, \S_{te}(0) \rangle \cdot  {{}{}{} m} \S_{te,r}(0)  \cdot 
\langle v_{\ell_2,r}, \S_j(0) \rangle \cdot   {{}{}{} m} \S_{j,r}(0)\\
[K(\tau)_{\ell_1,\ell_2}]_{te, j}:= &  \frac{1}{m} x_{te}^{\top} x_{j} \sum_{r=1}^{m} \langle v_{\ell_1, r}, \S_{te}(\tau) \rangle \cdot   {{}{}{} m} \S_{te,r}(\tau)  \cdot 
\langle v_{\ell_2,r}, \S_j(\tau) \rangle \cdot   {{}{}{} m} \S_{j,r}(\tau),
\end{align*}
where $K^*_{te}, K_{te}(\tau) \in \R^{d \times nd}$. Similarly, we have $K^*_i = [H^*]_i \in \R^{d \times nd}, K_i(\tau) = [H(\tau)]_i \in \R^{d \times nd}$ for training data $x_i$. 
Then, we define the kernel regression predictor.
\begin{definition}[NTK regression predictor] We define NTK regression predictor as
    \begin{align}\label{eq:f_ntk_form}
    F_{ntk}(\gamma(\tau), x_{te}) := & {{}{}{} m} K^*_{te} \gamma(\tau),
\end{align}
where $\gamma(\tau) \in \R^{nd}$ is the parameter  at timestamp $\tau$. 
\end{definition}
Recall that we have a training dataset ${\cal D}_n = \{ (x_i, y_i)\}_{i=1}^n$.
Then, we denote the corresponding objective function for $F_{ntk}$ as 
\begin{align}\label{eq:ntk_loss}
    \mathcal{L}_{ntk}(\gamma(\tau))=\frac{1}{2} \sum_{i=1}^n \|F_{ntk}(\gamma(\tau), x_i)-y_i \|_{2}^{2}.
\end{align}

Thus, based on Eq.~\eqref{eq:ntk_loss}, the gradient desent (GD) updating rule of $\gamma(\tau)$ is given by
\begin{align}\label{eq:69}
    \underbrace{ \gamma(\tau+1) }_{nd \times 1} = \underbrace{ \gamma(\tau) }_{nd \times 1} -\eta \cdot ( m \underbrace{ H^* }_{nd \times nd} \underbrace{ \gamma(\tau) }_{nd \times 1} - \underbrace{ \vect(Y) }_{nd \times 1} ), \quad \gamma(0)= {\bf 0}_{nd},
\end{align}
where the Eq.~\eqref{eq:69} is according to $\gamma(\tau+1)=\gamma(\tau)-\eta\nabla_{\gamma} \mathcal{L}_{ntk}(\gamma(\tau))$.

\subsection{Equivalence between Trained Net and Kernel Regression}\label{sec:equivalence_trained_net_kernel_regression}

We provide a stronger bound between $F_{ntk}$ and $F_{nn}$ result compared to Lemma F.1 in~\cite{arora2019exact}. Our following statement is stronger in the two following senses: their result only holds when $t \rightarrow \infty$, and our result holds for all $t \in [0,\infty)$; also their result only works for 1 dimension output space, our result holds arbitrary $d$ dimensional output space.

\begin{theorem}[Kernel value perturbation $\Rightarrow$ prediction perturbation] \label{thm:kernel_perturbation_to_prediction}
Fix $\epsilon_H \leq \frac{1}{2} \lambda$. If for all $\tau \geq 0$, $\|{K^*_{\ell, te}} - {K_{\ell, te}(\tau)}\|_F \leq \epsilon_{\ell, test}$ and $\|H^* - H(\tau)\|_F \leq \epsilon_H$, then for any $x_{te} \in \R^d$, $\ell \in [d]$ and $\tau \ge 0$, we have 
{{}{}{}
\begin{align*}
|F_{ntk}(\gamma(\tau), x_{te})_\ell - F_{nn}(W(\tau), x_{te})_\ell| \leq O\left(\frac{\sqrt{nd}}{\lambda} \epsilon_{\ell, test} + \frac{\sqrt{nd}}{\lambda^2} \log^2\left(\frac{nd}{\epsilon_H m\lambda }\right) \epsilon_H\right).
\end{align*}
}
\end{theorem}

\begin{proof}[Proof of Theorem~\ref{thm:kernel_perturbation_to_prediction}]
Our proof relies on a careful analysis of the trajectories induced by gradient flow for optimizing the neural network predictor $F_{nn}$ and the NTK predictor $F_{ntk}$. Then, we can have a similar argument to gradient descent at any timestamp $\tau$.

Recall that for any $x_{te}, x_i \in \mathbb{R}^d$, we have $K^*_{te}, K^*_i \in \R^{d \times nd}$ be the feature map induced by NTK. For any $x \in \R^d$, we define $\phi(x) \in \R^{d \times d}$ as following, for any $\ell \in [d] $, 
\begin{align*} 
\phi(x)_\ell = \frac{1}{\sqrt{m}} x \sum_{r=1}^{m} \langle v_{\ell, r}, \S(0) \rangle \cdot {{}{}{} m}\S_{r}(0). 
\end{align*}
{{}{}{} We denote $\phi(X) \in \R^{d \times nd}$ as the stack of feature map of $X \in \R^{d \times n}$.}

Note the optimal solution in  Eq.~\eqref{eq:f_ntk_form} can be rewritten as 
\begin{align*}
\min_\gamma \|\gamma\|_2 \text{ such that } {{}{}{} m} K^*_{i} \gamma = y_i \text{ for } i=1,\ldots,n.
\end{align*}
We have the optimal solution for kernel regression is $\gamma^* := {{}{}{} m^{-1} }(H^*)^{-1} \vect(Y)$
and its corresponding prediction for $x_{te}$ will be $F_{ntk}(\gamma(\tau), x_{te}) =  K^*_{te}(H^*)^{-1} \vect(Y)$.
The solution to this program can be rewritten as applying gradient flow on the
{{}{}{}
\begin{align*}
\min_\beta \sum_{i=1}^n  \|  \sqrt{m}  \phi(x_i)^\top \beta - y_i\|_2^2
\end{align*}
with initialization $\beta(0) = {\bf 0}_{d}$. We use $\beta(\tau)$ to denote this parameter at timestamp $\tau$ trained by gradient flow.
We denote 
\begin{align*}
    F_{ntk2}(\beta(\tau), x_{te}) : 
    = & ~ \sqrt{m}  \phi(x_{te})^\top \beta(\tau) 
\end{align*}
where $F_{ntk2}(\beta(\tau), x_{te})$ be the predictor for $x_{te}$ at time $\tau$. 
Then we have
\begin{align*}
    F_{ntk2}(\beta(\tau), x_{te})
    = & ~ \sqrt{m}  \underbrace{\phi(x_{te})^\top}_{\R^{d\times d}} \underbrace{\beta(\tau)}_{\R^{d}} \\
    = & ~ \sqrt{m}  \underbrace{\phi(x_{te})^\top}_{\R^{d\times d}} (\sqrt{m} \underbrace{\phi(X)}_{\R^{d\times nd}} ) \underbrace{\gamma(\tau)}_{\R^{nd}} 
     \\
    = & ~  m \underbrace{K^*_{te}}_{\R^{d\times nd}} \gamma(\tau)\\
    = & ~ F_{ntk}(\gamma(\tau), x_{te}) 
\end{align*}
where the second step follows $\beta(\tau) = \sqrt{m} \phi(X) \gamma(\tau)$ the third step follows $K^*_{te} = \phi(x_{te})^\top \phi(X)$.
}

With these notations, as $\tau$ goes to infinity, we denote, for any $\ell \in [d]$, 
\begin{align*}
F_{ntk2}(x_{te})_\ell = \int_{\tau=0}^{\infty} \frac{ \d F_{ntk2}(\beta(\tau), x_{te})_\ell}{ \d \tau} { \d \tau}
\end{align*}
where we have used the fact that the initial prediction is $0$ as $\beta(0) = {\bf 0}_{d}$. Similarly for $F_{nn}(x_{te})_\ell$. Let $F_{ntk2,i}(\tau) = F_{ntk2}(\beta(\tau), x_i)$ and $F_{ntk2}(\tau) \in \mathbb{R}^{d \times n}$. Similarly, for the NN predictor $F_{nn}$. 
Now we take a closer look at the time derivative:
\begin{align}\label{eq:f_ntk}
\frac{ \d F_{ntk2}(\beta(\tau), x_{te})_\ell}{ \d \tau} &= \left\langle \frac{\partial F_{ntk2}(\beta(\tau), x_{te})_\ell}{\partial \beta(\tau)}, \frac{\d\beta(\tau)}{ \d \tau}\right\rangle \notag\\
 &= \left\langle \frac{\partial F_{ntk2}(\beta(\tau), x_{te})_\ell}{\partial \beta(\tau)}, -\frac{\partial {\cal L}(\beta(\tau), \{x_i\}_{i=1}^n)}{\partial \beta(\tau)}\right\rangle\notag \\
&= - \left\langle  \frac{\partial F_{ntk2}(\beta(\tau), x_{te})_\ell}{\partial \beta(\tau)}, \sum_{i=1}^n\sum_{\ell_2=1}^d \left(F_{ntk2,i,\ell_2}(\tau) - y_{i,\ell_2}\right) \frac{\partial F_{ntk2}(\beta(\tau), x_i)_{\ell_2}}{\partial \beta(\tau)}\right\rangle\notag \\
 &= - m\left\langle  \phi(x_{te})_{\ell}, \sum_{i=1}^n\sum_{\ell_2=1}^d (F_{ntk2,i,\ell_2}(\tau) - y_{i,\ell_2})  \phi(x_i)_{\ell_2} \right\rangle \notag\\
 &= - m\vect({K^*_{\ell, te}})^\top \vect(F_{ntk2}(\tau) - Y)
\end{align}
where the first step follows from simple algebra, the second step follows from ODE formulation (we remark that this is a very standard step in all the NTK literature), the third step follows from Eq.~\eqref{eq:ntk_loss}, the fourth step follows from the definition of $\phi(x_{te})_{\ell}$, the last step follows from simple algebra. 
 
We can obtain a time derivative of the same form for $F_{nn}$. 

\begin{align}\label{eq:f_nn}
\frac{ \d F_{nn}(W(\tau), x_{te})_\ell}{ \d \tau} &= \left\langle \frac{\partial F_{nn}(W(\tau), x_{te})_\ell}{\partial W(\tau)}, \frac{dW(\tau)}{ \d \tau}\right\rangle \notag\\
 &= \left\langle \frac{\partial F_{nn}(W(\tau), x_{te})_\ell}{\partial W(\tau)}, -\frac{\partial {\cal L}(W(\tau), \{x_i\}_{i=1}^n)}{\partial W(\tau)}\right\rangle\notag \\
 &= -\left\langle \frac{\partial F_{nn}(W(\tau), x_{te})_\ell}{\partial W(\tau)}, \sum_{i=1}^n\sum_{\ell_2=1}^d (F_{nn,i, \ell_2}(\tau) - y_{i,\ell_2}) \frac{\partial F_{nn}(W(\tau), x_i)_{\ell_2}}{\partial W(\tau)} \right\rangle\notag \\
 &= -m\vect({K_{\ell, te}(\tau)})^\top \vect(F_{nn}(\tau) - Y)
\end{align}
where the first step follows from simple algebra, the second step is standard in NTK literature, the third step follows from Eq.~\eqref{eq:ntk_loss}, the last step follows from simple algebra.

Thus we analyze the difference between the NN predictor and NTK predictor via this integral form
\begin{align*}
& |F_{nn}(x_{te})_\ell - F_{ntk2}(x_{te})_\ell| \\
 &= \left|F_{nn}(W(0), x_{te})_\ell + \int_{\tau=0}^\infty \left(\frac{ \d F_{nn}(W(\tau), x_{te})_\ell}{ \d \tau} - \frac{ \d F_{ntk2}(\beta(\tau), x_{te})_\ell}{ \d \tau}\right){ \d \tau}\right| \\
 &= |F_{nn}(W(0), x_{te})_\ell| + \left|- m\int_{\tau=0}^\infty \left(\vect({K_{\ell, te}(\tau)})^\top \vect(F_{nn}(\tau) - Y) - \vect({K^*_{\ell, te}})^\top \vect(F_{ntk2}(\tau) - Y)\right){ \d \tau}\right| \\
 &=  \left|- m \int_{\tau=0}^\infty \left(\vect({K_{\ell, te}(\tau)})^\top \vect(F_{nn}(\tau) - Y) - \vect({K^*_{\ell, te}})^\top \vect(F_{ntk2}(\tau) - Y)\right){ \d \tau}\right| \\
 &\leq  m \left|\int_{\tau=0}^\infty \vect({K_{\ell, te}(\tau)} - {K^*_{\ell, te}})^\top \vect(F_{nn}(\tau)-Y) { \d \tau}\right| + m \left|\int_{\tau=0}^\infty \vect({K^*_{\ell, te}})^\top \vect(F_{nn}(\tau) - F_{ntk2}(\tau)) { \d \tau}\right| \\
 &\leq  m \max_{0 \leq t \leq \infty} \|{K_{\ell, te}(\tau)} - {K^*_{\ell, te}}\|_F \int_{\tau=0}^\infty \|F_{nn}(\tau) - Y \|_F { \d \tau} +  m \max_{0 \leq t \leq \infty} \|{K^*_{\ell, te}}\|_F \int_{\tau=0}^\infty \|F_{nn}(\tau) - F_{ntk2}(\tau)\|_F { \d \tau} \\
 &\leq  m \epsilon_{\ell, test} \int_{\tau=0}^\infty \|F_{nn}(\tau) - Y \|_F { \d \tau} + m \max_{0 \leq t \leq \infty} \|{K^*_{\ell, te}}\|_F \int_{\tau=0}^\infty \|F_{nn}(\tau) - F_{ntk2}(\tau)\|_F { \d \tau},
\end{align*}
where the first step follows from the difference between the NN predictor and NTK predictor, the second step follows from  Eq.~\eqref{eq:f_ntk} and Eq.~\eqref{eq:f_nn}, the third step follows $|F_{nn}(W(0), x_{te})_\ell| = 0$ by symmetric initialization from Definition~\ref{def:duplicate_weights}, the fourth step follows from simple algebra, the fifth step follows from Frobenius norm, the last step follows from simple algebra. 

For the first term, recall $\|H^* - H(\tau)\|_F \leq \epsilon_H$ and, by Claim~\ref{cla:C1}, we have \begin{align*}
\lambda_{min}(H(\tau)) \geq \frac{1}{2} \lambda.
\end{align*}

Using this fact we know $\|F_{nn}(\tau)-Y \|_F \leq \exp (-\frac{m}{2} \lambda \tau ) \|F_{nn}(0)-Y \|_F$ (The reason to obtain this is due to solve ODE).

Therefore, by Lemma~\ref{lem:bound_init_loss}, we can bound
\begin{align*}
\int_{\tau=0}^\infty \|F_{nn}(\tau)-Y \|_F { \d \tau} 
= & ~ \int_{\tau=0}^\infty \exp\left(-\frac{m}{2} \lambda \tau\right) \|F_{nn}(0)-Y \|_F { \d \tau} \\
 = & ~  O (\frac{\sqrt{nd}}{ m\lambda}).
\end{align*}

To bound $\int_{\tau=0}^\infty \|F_{nn}(\tau) - F_{ntk2}(\tau)\|_F \d \tau$, we observe that $F_{nn}(\tau) \rightarrow y$ and $F_{ntk2}(\tau) \rightarrow y$ with linear convergence rate. Therefore, we can choose some $\tau_0 = \frac{C}{m\lambda } \log\left(\frac{nd}{\epsilon_H \cdot m\lambda }\right)$ so that

\begin{align*}
\int_{\tau_0}^\infty \|F_{nn}(\tau) - F_{ntk2}(\tau)\|_F { \d \tau} &\leq \int_{\tau_0}^\infty \|F_{nn}(\tau) - Y \|_F { \d \tau} + \int_{\tau_0}^\infty \|F_{ntk2}(\tau) - Y \|_F { \d \tau} \\
 &\leq O\left(\frac{1}{m\lambda } (\|F_{nn}(\tau_0) - Y \|_F + \|F_{ntk2}(\tau_0) - Y \|_F)\right) \\
 &\leq O\left(\frac{\sqrt{nd}}{m\lambda} \exp\left(-m\lambda  \tau_0\right)\right) \\
 &\leq O(\epsilon_H).
\end{align*}
where the first step follows from simple algebra, the second step follows from integral range is $\tau_0$, the third step follows from Lemma~\ref{lem:bound_init_loss}, the last step follows from choice of $\tau_0$.

Thus it suffices to bound $\int_{\tau=0}^{\tau_0} \|F_{nn}(\tau) - F_{ntk2}(\tau)\|_F { \d \tau} \leq \tau_0 \max_{0 \leq t \leq \tau_0} \|F_{nn}(\tau) - F_{ntk2}(\tau)\|_F$.

First observe that
\begin{align*}
\|F_{nn}(\tau) - F_{ntk2}(\tau)\|_F &\leq \|F_{nn}(0)\|_F + \int_{s=0}^\tau \left\|\frac{ \d (F_{nn}(s) - F_{ntk2}(s))}{ \d s}\right\|_F  \d s \\
 & =  \int_{s=0}^\tau \left\|\frac{ \d (F_{nn}(s) - F_{ntk2}(s))}{ \d s}\right\|_F  \d s,
\end{align*}
where the last step follows symmetric initialization from Definition~\ref{def:duplicate_weights}. 

Note
\begin{align*}
\frac{ \d (F_{nn}(\tau) - F_{ntk2}(\tau))}{ \d \tau} &= - m H(\tau) \vect(F_{nn}(\tau) - Y) + m H^* \vect(F_{ntk2}(\tau) - Y) \\
 &= - m H^* \vect(F_{nn}(\tau) - F_{ntk2}(\tau)) +  m (H^* - H(\tau)) \vect(F_{nn}(\tau) - Y)
\end{align*}
where the first step follows from definition of $F_{nn}$ and $F_{ntk2}$.

Since $H^*$ is positive semidefinite, $-H^* \vect(F_{nn}(\tau) - F_{ntk2}(\tau))$ term only makes $\|F_{nn}(\tau) - F_{ntk2}(\tau)\|_F$ smaller. Therefore, we have
\begin{align*}
\|F_{nn}(\tau) - F_{ntk2}(\tau)\|_F &\leq  m\int_{s=0}^\tau \|F_{nn}(s) - Y \|_F \|H(\tau) - H^*\|_F  \d s \\
 &\leq m\tau  \|F_{nn}(0) - Y \|_F \epsilon_H \\
 &\leq O\left(m \tau\sqrt{nd}\epsilon_H\right),
\end{align*}
where the last step is by Lemma~\ref{lem:bound_init_loss}.

Therefore, we have
\begin{align*}
\int_{\tau=0}^{\tau_0} \|F_{nn}(\tau) - F_{ntk2}(\tau)\|_F { \d \tau} &\leq O\left(m\tau_0^2 \sqrt{nd} \epsilon_H\right) \\
 &= O\left(\frac{\sqrt{nd}}{m\lambda^2 } \log^2 \left(\frac{nd}{\epsilon_H m\lambda }\right) \epsilon_H\right).
\end{align*}
where the first step follows from integral range is $\tau_0$, the second step follows from the choice of $\tau_0$.

Lastly, as $F_{ntk2}(x_{te})_\ell = F_{ntk}(x_{te})_\ell$, we put things together and get 
\begin{align*}
|F_{ntk}(x_{te})_\ell - F_{nn}(x_{te})_\ell| \leq O\left(\frac{\sqrt{nd}}{\lambda} \epsilon_{\ell, test} + \frac{\sqrt{nd}}{\lambda^2} \log^2\left(\frac{nd}{\epsilon_H m\lambda }\right) \epsilon_H\right).
\end{align*}

From the above, after we change the integration from $(0,\infty)$ to $(0,\tau)$, the statement still holds. Then, based on the gradient flow version, we can have a gradient descent version with a constant error factor by replacing integral with geometric summarization (for example $\sum_{i=0}^{\infty} a^{i} < 2$, when $a \in (0, 0.5)$ ). 
\end{proof}

\section{Diffusion}\label{sec:diffusion}

In Section~\ref{app:diff_main}, we provide the proof of our main result of diffusion. In Section~\ref{app:tools_previous_works}, we provide some tools from previous works.

We first define an auxiliary function $\tilde{F}_{ntk}$ of the same functional form as $F_{ntk}$, but trained on a pseudo dataset $\tilde{S}:=\{\tilde{y}_i, x_i \}_{i=1}^{n}$ 
with $\tilde{y}_i := F_\mathcal{H} (x_i) + \epsilon_i $ and $ \epsilon_i := y_i - F_* (x_i)$. Then, we have the following claim. 
\begin{claim}[Loss decomposition]\label{cla:decompose}
    We can decompose our target function as the following  
\begin{align*}
    \frac{1}{T} \int_{0}^{T} \mathbb{E} [ \| F_{nn}(W(\tau), (t, x(t))) -F_{*} (t, x(t)) \|_{2}^{2} ] \mathrm{d} t \le Z_1 + Z_2 + Z_3 + Z_4,
\end{align*}
where 
\begin{align*}
    Z_1 = & \frac{1}{T} \int_{0}^{T} \mathbb{E} [ \|F_{nn}(W(\tau), (t, x(t)))  - F_{ntk}(\gamma(\tau), (t,x(t)) ) \|_{2}^{2}  ] \mathrm{d} t & \text { (coupling) }\\
    Z_2 = & \frac{1}{T} \int_{0}^{T} \mathbb{E} [ \|F_{ntk}(\gamma(\tau), (t,x(t)) ) - \tilde{F}_{ntk}(\gamma(\tau), (t,x(t)) ) \|_{2}^{2}  ] \mathrm{d} t & \text { (label mismatch) } \\
    Z_3 = & \frac{1}{T} \int_{0}^{T} \mathbb{E} [ \|\tilde{F}_{ntk}(\gamma(\tau), (t,x(t)) ) -F_{\mathcal{H}} (t, x(t) ) \|_{2}^{2}  ] \mathrm{d} t  & \text { (early stopping) } \\
    Z_4 = & \frac{1}{T} \int_{0}^{T} \mathbb{E} [ \|F_{\mathcal{H}} (t, x(t) )-F_{*} (t, x(t) ) \|_{2}^{2}  ] \mathrm{d} t .  & \text {(approximation)}.
\end{align*}
\end{claim}

The coupling error term is the gap between neural networks $F_{nn}$ and a kernel function $F_{ntk}$. The approximation error term is the gap between the target function $F_*$ and its corresponding RKHS function $F_H$. These two terms transfer the problem of neural networks training into the problem of kernel regression. 
\subsection{Main Result of Diffusion}\label{app:diff_main}

In this section, we prove the main result of diffusion.

\begin{theorem}[Restatement of Theorem~\ref{thm:diff_main}]\label{thm:diff_main_formal}
Suppose Assumptions~\ref{ass:target_density_function},~\ref{ass:smallest_eigenvalue_probability_bound_38},~\ref{ass:function_g},~\ref{ass:function_f_lipschitz_x} hold and we set {{}{}{}$ m = \Omega( \lambda^{-2} n^3 d^3 \exp(18B)\log^2(nd/\delta)  )$} and {{}{}{} $\eta = 0.1 \lambda / (m n^2 d^2 \exp(16B))  $}.  
Moreover, suppose $\widehat{T}$ satisfies Assumption~\ref{ass:kernel_regression_error_bound} with corresponding $\epsilon(n, \widehat{T})$. Then for large enough $R_{\mathcal{H}}$, with probability at least $1-\delta$, it holds that
\begin{align*}
    & \frac{1}{T} \int_{0}^{T} \lambda(t) \mathbb{E} [ \|s_{W(\widehat{T})} (t, x(t) )-\nabla \log p_{t} (X_{t} ) \|_{2}^{2} ] \mathrm{d} t \\
& \leq O\left( \frac{1}{\lambda \sqrt{n}}  + \epsilon(n, \widehat{T}) + d A^{2} (R_{\mathcal{H}}) + d A (R_{\mathcal{H}})+\sqrt{d A (R_{\mathcal{H}}) \Gamma_{\delta}}+\Gamma_{\delta} \right).
\end{align*}
\end{theorem}
\begin{proof}[Proof of Theorem~\ref{thm:diff_main}]
Note that the $m$ and $\eta$ satisfy the conditions in Theorem~\ref{thm:formal}. The reason about a different $m$ is that we choose a different $R$ and apply Lemma~\ref{lem:loss_one_step_shrinking} one more time. 
Recall the $\epsilon_{\ell, test}$ and $\epsilon_H$ are defined in Theorem~\ref{thm:kernel_perturbation_to_prediction}. 

Note that $H^*=H(0)$. By Lemma~\ref{lem:perturb_w}, Part 2, let {{}{}{} $R = \lambda/(2n^2 d^2 \exp(10B))$}, we have with probability at least $1-\delta$ such that
\begin{align*}
\| \underbrace{ H^* }_{nd \times nd} - \underbrace{ H(\tau) }_{nd \times nd} \|_F \leq \epsilon_H = {{}{}{} \frac{\lambda}{2nd}.}
\end{align*}
Note that $K^*_{\ell, te}$ and $K_{\ell, te}$ share the same weight perturbation as $H^*$ and $H(\tau)$. Thus, by using the same proof as Lemma~\ref{lem:perturb_w}, Part 1, we have
\begin{align*}
\| \underbrace{ K^*_{\ell, te}}_{n \times d} - \underbrace{ K_{\ell, te}}_{n \times d} \|_F \leq \epsilon_{\ell, test} = {{}{}{} \frac{\lambda}{2n^{1.5}d^{1.5}}.}
\end{align*}
We have 
\begin{align*}
& \|F_{ntk}(\gamma(\tau), x_{te}) - F_{nn}(W(\tau), x_{te})\|_2 \\
\le &~ \sqrt{d} \max_{\ell \in d} |F_{ntk}(\gamma(\tau)_\ell, x_{te}) - F_{nn}(W(\tau), x_{te})_\ell| \\
\leq &~ O\left(\frac{\sqrt{n}d}{\lambda} \max_{\ell \in [d]}\epsilon_{\ell, test} + \frac{\sqrt{n}d}{\lambda^2} \log^2\left(\frac{nd}{\epsilon_H m \lambda }\right) \epsilon_H\right)\\
\le &~ O\left(\frac{\sqrt{n}d}{\lambda} {{}{}{} \frac{\lambda}{n^{1.5}d^{1.5}}} + \frac{\sqrt{n}d}{\lambda^2} {{}{}{}\log^2\left(\frac{nd}{ m\lambda}\right) \frac{\lambda}{nd}} \right)\\
\le &~ O\left(\frac{1}{\lambda\sqrt{n}} {{}{}{}\log^2\left(\frac{nd}{ m\lambda}\right)} \right)\\
\le &~ O\left(\frac{1}{\lambda\sqrt{n}} \right)
\end{align*}
where the first step follows from simple algebra, the second step is by Theorem~\ref{thm:kernel_perturbation_to_prediction}. 

Thus, we finish the proof by Claim~\ref{cla:decompose}, where coupling is from above,  label mismatch is from Theorem~\ref{the:label_mismatch_error_bound_310}, early stopping is from Assumption~\ref{ass:kernel_regression_error_bound} and approximation is from Theorem~\ref{the:universal_approximation_score_function_36}.
\end{proof}

\subsection{Tools From Previous Works}\label{app:tools_previous_works}

We have the following statements from previous works \cite{hrx24}.

\begin{theorem}[Theorem 3.6 in \cite{hrx24}, universal approximation of score function]\label{the:universal_approximation_score_function_36}
Suppose Assumptions~\ref{ass:target_density_function},~\ref{ass:function_g} and~\ref{ass:function_f_lipschitz_x} hold. Let $R_{\mathcal{H}}$ be larger than a constant $c_{1}$, i.e., $C(d + 1, 0)$ in Proposition 6 of~\cite{bach2017breaking}, which depends only on $d$. There exists a function $F_{\mathcal{H}} \in \mathcal{H}$ such that $ \|F_{\mathcal{H}} \|_{\mathcal{H}}^{2} \leq d R_{\mathcal{H}}$ and
\begin{align*}
    \frac{1}{T} \int_{0}^{T} \mathbb{E} [ \|F_{\mathcal{H}} (t, x(t) )-F_{*} (t, x(t) ) \|_{2}^{2} ] \mathrm{d} t \leq d A^{2} (R_{\mathcal{H}}).
\end{align*}
\end{theorem}

\begin{theorem}[Theorem 3.10 in \cite{hrx24}, label mismatch]\label{the:label_mismatch_error_bound_310} Suppose Assumptions~\ref{ass:target_density_function} and~\ref{ass:smallest_eigenvalue_probability_bound_38} hold. If we initialize both $F_{ntk}$ and $\tilde{F}_{ntk}$ properly, then with probability at least $1-\delta$ it holds simultaneously for all $\tau$ that
\begin{align*}
     & ~ \frac{1}{T} \int_{0}^{T} \mathbb{E} [ \|F_{ntk}(\gamma(\tau), (t,x(t)) )-\tilde{F}_{ntk}(\gamma(\tau), (t,x(t)) ) \|_{2}^{2} ] \mathrm{d} t \\
    & ~ \leq d A (R_{\mathcal{H}})+C_{0} (\sqrt{d A (R_{\mathcal{H}}) \Gamma_{\delta}}+\Gamma_{\delta} )
\end{align*}
where $C_{0}$ is a constant defined in Theorem 1 of~\cite{reeve2020optimistic}.
\end{theorem}

\section{Discussion}\label{sec:discussion}

In this section, we provide discussions about the potential extensions of our method on various popular frameworks, such as attention mechanism (Section~\ref{sec:discussion:self_attention_learning}) and feature learning (Section~\ref{sec:discussion:feature_learning}).

\subsection{Self-attention Learning} \label{sec:discussion:self_attention_learning}
The self-attention can be written as 
\begin{align}\label{eq:self_attention}
    F(W^K X, W^QX, W^VX) \in \R^{d \times n'},
\end{align}
where $W^K, W^Q, W^V \in \R^{d\times d}$ denotes key, query, and value matrix respectively and $X\in \R^{d \times n'}$ is a sequence of $n'$ tokens. As our work is a first step to understanding softmax, it is natural to consider how to extend our results to self-attention. It is well-known that using two reformulation tricks: tensor-trick and SVM-trick \citep{gswy23,gsx23,as24_arxiv}, any analysis for softmax function can be naturally generalized to attention function $F(W^K X, W^QX, W^VX)$. Therefore, we conjecture that we can borrow the idea from~\cite{gswy23,gsx23,as24_arxiv} to decouple Eq~\eqref{eq:self_attention} into the value term and the softmax term. And, we can alternatively optimize the weights for the softmax term ($W^k,W^Q$) and the value term ($W^V$). 
We leave this valuable direction as a future work.

\subsection{Feature Learning} \label{sec:discussion:feature_learning}

Recently, there is a line of work showing that feature learning may be beyond NTK on sample complexity or time complexity, e.g.,~\cite{al19, wllm19,hn19,all19,dm20, cbl+20, yh20,hy20,lmz20,gmmm20,rgkz21, mkas21,lxmz21,dls22,swl22,swl24} and many more. It is worth studying the feature learning ability of two-layer softmax NN to figure out what feature pattern the softmax prefers to learn and how it happens. We leave this valuable direction as a future work.  

\end{document}